\documentclass[12pt]{article}
\usepackage{ulem}
\usepackage{graphicx} 
\usepackage[top=2cm, bottom=2cm, left=1.5cm, right=1.5cm]{geometry}

\usepackage{subcaption}
\usepackage{booktabs}
\usepackage{algorithm}
\usepackage{algpseudocode}
\usepackage{algpseudocode}
\usepackage[percent]{overpic}
\usepackage{amsmath, amssymb, amsfonts, amstext, amsthm, textcomp, enumerate, xcolor}
\usepackage[mathscr]{euscript}
\usepackage{microtype}
\usepackage{centernot}

\newtheorem{thm}{Theorem}[section]
\newtheorem{lem}[thm]{Lemma}
\newtheorem{rem}[thm]{Remark}

\usepackage{color}

\begin{document}
\begin{center}
\vspace{3cm}
 {\bf \large A Theoretical and Experimental Study of a Novel Adaptive Learning Algorithm}\\
\vspace{1cm}

{\bf Sakshi Kumari$^{\#}$, Shyam Kumar M$^{*}$, Sushmitha P$^{\#}$}\\
\vspace{.5cm}
$^{\#}$Department of Mathematics\\
Indian Institute of Technology Patna\\
Bihta 801 106\\
India\\

\vspace{.5cm}
$^{*}$Department of Mechanical Engineering\\
Indian Institute of Technology Kharagpur\\
Kharagpur  \\
India 
\end{center}
\vspace{1cm}
\begin{center}
    {\bf ABSTRACT}
\end{center}
A crucial component of machine learning algorithms is minimizing loss functions with less computational cost and less oscillations. While adaptive learning rate-based optimizers have been widely used for real-world tasks, they do not guarantee convergence, which is why AMSGrad was later introduced to investigate the non-convergence behaviour of Adam. In this paper, popular adaptive optimization methods like Adam and AMSGrad are critically reviewed with an emphasis on their fundamental design concepts. To address limitations of the above mentioned optimizers, a new optimizer variant, C-Adam, is proposed based on the line of sight approach. A theoretical proof for convergence is also provided and the optimizer is validated through a number of real-life based numerical experiments.
\vspace{2cm}
 
\vspace{2cm}


\newpage

\section{Introduction}\label{introduction}
The minimization of objective functions subject to certain constraints is a central problem in optimization theory and lies at the core of many ongoing developments in Artificial Intelligence \cite{avramelou2025sign, wu2026convolutional, das2026meaning}. For example, solving systems of non-linear partial differential equations using the highly non-convex function leads to several applications in the field of science and engineering \cite{yu2022gradient}.

An optimizer is an algorithm used to update the internal parameters of a model, such as weights and biases, and provide a point, called optimial point, that minimizes the loss function, which in turn works as an objective function in machine learning. There are various optimizers in use, of which the most relevant ones are discussed here. Stochastic Gradient Descent (SGD) has emerged as one of the most widely used optimizers, where gradients are approximated using randomly selected subset of data. Given a convex function $f$ with unbiased stochastic gradients, bounded variance, and step sizes that satisfy the Robbins-Monro conditions, namely $\sum_{t} \alpha_t = \infty \quad \text{and} \quad \sum_{t} \alpha_t^2 < \infty$, the SGD converges almost surely (that is with probability 1) to the optimal point \cite{NguyenNguyenRichtarikScheinbergTakacVanDijk2019, liu2023aiming}.

In practice, SGD tends to search for an optimal point near local minima as weight updates are purely based on the current gradient of the loss function. This led to the introduction of the optimizer SGD with momentum to counter the SGD drawback by adding a velocity term that stores information about past gradients. Another limitation of SGD is that the learning rate is fixed, i.e., there is no control over individual parameter learning. AdaGrad \cite{DuchiHazanSinger2011} solves this limitation by adjusting the learning rate for each parameter based on the history of past gradients observed. AdaGrad tends to make adaptive learning rate strictly decreasing in nature by scaling the learning rate by a factor of $(\sum_{t} g_t^2)^{-1/2}$, which raises yet another concern of learning rate decay in the case of dense gradients. To overcome this challenge of AdaGrad, RMSProp \cite{TielemanHinton2012} and AdaDelta \cite{zeiler2012} were introduced. In order to address these issues simultaneously, Adam \cite{KingmaBa2014} was later introduced as a mix of RMSProp and SGD with momentum. 
This made Adam effective in adjusting the adaptive learning rate for individual parameters and in navigating complex, noisy, or non-stationary landscapes without being stuck at the local minima.

Despite being one of the most popular optimizers \cite{sreckovic2025your, kalra2021applying}, Adam, with its own demerits, has shown a non-convergence nature due to the non-monotonic characteristics of the second order moment \cite{ReddiKaleKumar2019}. To ensure this, AMSGrad was introduced \cite{ReddiKaleKumar2019} retaining the maximum of second order moment observed and by scaling the learning rate using the past seen second moments during parameter updates. Inspite of resolving certain theoretical issues of Adam, AMSGrad was shown to lead to poor generalization in case of extreme learning rate \cite{luo2019adaptive}. AdaBound was then introduced, in which dynamic bounds on adaptive learning rates were applied to avoid extreme updates \cite{luo2019adaptive}. In AdaBound, a combination of Adam and SGD was used leading to early convergence and better generalization. The authors in \cite{zaheer2018adaptive} developed a method to address convergence and generalization issues in previous optimizers like RMSProp, Adam, AdaDelta. They then went on to develop a new adaptive optimization algorithm, YOGI, that controls the increase in effective learning rate for better performance, maintaining the theoretical guarantee of convergence. AMSGrad and its variants are widely used when stable convergence and reliable optimization are required in deep neural networks, and have shown effective results in several problem domains \cite{svelebamultilayered,mohammed2019new}. However, by forcing the second moment to be non-decreasing, AMSGrad might tend to keep the learning rate too small for too long, hindering the ability to escape saddle points or navigate complex, non-convex surfaces. This leads to a fundamental problem in the fix itself, which renders it too cautious.

In search of an optimizer that guarantees the non-decreasing character of the second moment while limiting the decay in the adaptive learning rate caused by noisy gradients, a customized version of Adam, namely C-Adam, is proposed, that relaxes the aggressive maximization of the previously observed second moments in AMSGrad while ensuring the non-decreasing nature of the same. The now proposed optimizer tries to incorporate the theoretical correctness of AMSGrad with enhanced practical applications. To deliver this, we shift from point-wise update to ``line of sight" update, ensuring that the optimizer accounts for the transitionary gradients between discrete steps. The structure of the paper is as follows. In Section \ref{prelim}, all the preliminaries and notations are given. In Section \ref{math}, the algorithm of C-Adam is discussed and the regret bound for the proposed optimizer is derived, mathematically ensuring the performance is either better or comparable with the existing optimizers. Examples where the suggested optimizer arrives at a final optimal point with nearly zero regret while both Adam and AMSGrad are unable to learn the solution surface, are also provided. The performance of C-Adam, in comparison with Adam and AMSGrad, on various real-life based examples are demonstrated in Section \ref{exp}. Concluding remarks and the possible drawbacks of C-Adam are discussed in Section \ref{conclusion}.

\section{Preliminaries}\label{prelim}
Throughout the paper, matrix $A$ denotes a positive definite $d\times d$ matrix, unless otherwise mentioned. Similarly for $x\in \mathbb{R}^d$ denotes a $d$-dimensional vector. $\sqrt{x}$ and $x^2$ are used to denote the vectors $(\sqrt{x_i})_i$ and $(x_i^2)_i$, the entrywise square root and square respectively of the vector $x$. Given any two vectors $x,y \in \mathbb{R}^d$, max$\{x,y\}$ represents the vector (max$\{x_i,y_i\})_i$. In a similar fashion, $x\leq y$ or $x=y$ represent the entrywise inequality and equality of the vectors.

$A^{\frac{1}{2}}$ is used to represent $\sqrt{A}$, which is the matrix all of whose entries are the square roots of the corresponding entries of $A$. $\Vert A_i\Vert_2$ is used to denote the $l_2$-norm of the $i^{th}$ row of $A$. The projection operation $\Pi_{\Omega, A}(y)$, where $\Omega$ is a set with bounded diameter $\Omega_\infty$, is defined as $\Pi_{\Omega, A}(y)=\text{arg min}_{x\in \Omega}\Vert A^{\frac{1}{2}}(x-y)\Vert$, for $y\in \mathbb{R}^d$.

In the online optimization setup, at each time step $t$, the algorithm picks a vector (of parameters of the model to be learned) $x_t \in \Omega$, where $\Omega \subset \mathbb{R}^d$ is the feasible set of points. The algorithm then uses a loss function $f_t$ and calculates $f_t(x_t)$. Then the regret of the algorithm after $T$ rounds of this process is given by $R_T = \sum_{i=1}^T f_t(x_t) - \text{min}_{x\in \Omega} \sum_{i=1}^T f_t(x)$.

Throughout this paper, we assume that $\Omega$ is a set with a bounded diameter $\Omega_\infty$ and the $l^\infty$ norm of the gradient function, $\nabla f_t$ (which we shall denote by $g_t$), $\Vert \nabla f_t(x)\Vert_\infty$ is bounded for all $t\in [T]$ and $x\in \Omega$ with bound $M$, that is $\Vert \nabla f_t(x)\Vert_\infty\leq M$.

\section{Revisiting the Convergence of AMSGrad : Introduction of C-Adam}\label{math}
Adam is one of the best known optimizers among various other optimizers discussed in Section \ref{introduction}. 

It is observed that Adam (refer Algorithm \ref{adam algo}) uses the unbiased second-order moment estimation of the gradient to calculate the adaptive learning rate. 

\begin{algorithm}[H]
\caption{Adam}
\label{adam algo}
\begin{algorithmic}[1]

\Require $x_1 \in \mathcal{F}$, step sizes $\{\alpha_t\}_{t=1}^T$, $\{\beta_{1t}\}_{t=1}^T$, $\beta_2$
\State Initialize $m_0 = 0$, $v_0 = 0$, $\bar{v}_0 = 0$

\For{$t = 1$ to $T$}
    \Statex $g_t = \nabla f_t(x_t)$
    \Statex $m_t = \beta_{1t} m_{t-1} + (1 - \beta_{1t}) g_t$
    \Statex $v_t = \beta_2 v_{t-1} + (1 - \beta_2) g_t^2$
    \Statex $\bar{m}_t = \dfrac{m_t}{1 - \beta_1^t}$
    \Statex $\bar{v}_t = \dfrac{v_t}{1 - \beta_2^t}$
    \Statex $\bar{V}_t = \mathrm{diag}(\bar{v}_t)$
    \Statex $x_{t+1} = \Pi_{\mathcal{F}, \sqrt{\bar{V}_t}} \left( x_t - \alpha_t \frac{\bar{m}_t}{\sqrt{\bar{v}_t + \epsilon}} \right)$
\EndFor
\end{algorithmic}
\end{algorithm}
Adam uses first moment $m_{t} = (1-\beta_{1})\sum_{i=1}^{t}{\beta_{1}^{t-i}g_{i}}$ and second moment $v_{t} = (1-\beta_{2})\sum_{i=1}^{t}{\beta_{2}^{t-i}g_{i}^{2}}$ to update the parameter $x_{t}$ as shown in Algorithm \ref{adam algo}, where $\beta_{1}$ and $\beta_{2}$ are hyperparameters.  $\left\{\beta_{1}, \beta_{2}\right\}$ are typically chosen to be $\left\{0.9, 0.999\right\}$ for satisfactory practical results. More generally any  $\left\{\beta_{1}, \beta_{2}\right\}$ with $\dfrac{\beta_{1}}{\sqrt{\beta_{2}}} < 1$, where $\beta_{1} \in (0,1)$ and $\beta_{2} \in (0,1)$, satisfies the required conditions for Algorithm \ref{adam algo} to converge to the optimal point. 

Whereas the AMSGrad algorithm (refer Algorithm \ref{ams algo}) uses the maximum of the previous unbiased second-order moment estimations of the gradient to calculate the adaptive learning rate. This way AMSGrad prevents the increase in the adaptive learning rate as the number of iterations increase and ensures the convergence of the model. 
\begin{algorithm}[H]
\caption{AMSGrad}
\label{ams algo}
\begin{algorithmic}[1]

\Require $x_1 \in \mathcal{F}$, step sizes $\{\alpha_t\}_{t=1}^T$, $\{\beta_{1t}\}_{t=1}^T$, $\beta_2$
\State Initialize $m_0 = 0$, $v_0 = 0$, $\hat{v}_0 = 0$, $\epsilon = 1e^{-8}$

\For{$t = 1$ to $T$}
    \Statex $g_t = \nabla f_t(x_t)$
    \Statex $m_t = \beta_{1t} m_{t-1} + (1 - \beta_{1t}) g_t$
    \Statex $v_t = \beta_2 \hat{v}_{t-1} + (1 - \beta_2) g_t^2$
    \Statex $\hat{v}_t = \max(v_{t}, \hat{v}_{t-1})$
    \Statex $\hat{V}_t = \mathrm{diag}(\hat{v}_t)$
    \Statex $x_{t+1} = \Pi_{\mathcal{F}, \sqrt{\hat{V}_t}} \left( x_t - \alpha_t \frac{\hat{m}_t}{\sqrt{\hat{v}_t + \epsilon}} \right)$
\EndFor
\end{algorithmic}
\end{algorithm}

AMSGrad was proposed as a variant of Adam with an improved convergence guarantee. Similar to Adam, AMSGrad computes first moment and second moment as $m_{t} = (1-\beta_{1})\sum_{i=1}^{t}{\beta_{1}^{t-i}g_{i}}$ and $v_{t} = (1-\beta_{2})\sum_{i=1}^{t}{\beta_{2}^{t-i}g_{i}^{2}}$ respectively, where $\beta_{1} \in (0,1)$ and $\beta_{2} \in (0,1)$ denote the exponential decay rates for the first and second moments. Similar to Adam, $\left\{\beta_{1}, \beta_{2}\right\} = \left\{0.9, 0.999\right\}$ is chosen for satisfactory performance. However, instead of updating $x_{t}$ using $v_{t}$, AMSGrad maintains the maximum of all past second moment estimates as $\hat{v}_{t} = \max(\hat{v}_{t-1}, v_{t})$, and updates $x_{t}$ as $x_{t+1} = x_t-\alpha_{t}\dfrac{ m_t}{\sqrt{\hat{v}_{t}+\epsilon}}$ as shown in Algorithm \ref{ams algo}.

AMSGrad was later found to have drawbacks \cite{HuangWangDong2018, TanYinLiuWan2019}. Since there is no bound on the maximum constraints of AMSGrad, the following issues arise:
\begin{enumerate}
    \item In the case when the maximum constraint is very high, the adaptive learning rate will be over conservative, which directly results in the slow convergence of AMSGrad. For example, in case of non-linear phenomena or noisy datasets, the gradient fluctuates repeatedly. Since AMSGrad forces $\hat{v_{t}}$ to be the $\max(v_{t}, \hat{v}_{t-1})$, the adaptive learning rate $\alpha_{t}/\sqrt{\hat{v_t} + \epsilon}$ is forced to be very small, resulting in slow convergence of the loss function as seen in Fig. \ref{noisy_data_result_comparison}. \label{issue1}

   \item The case when the maximum constraint is very low, the adaptive learning rate will be very high and AMSGrad will face oscillations in the losses. But this being a generic issue, we focus on attempting to rectify the problems arising from the previous issue. \label{issue2}
\end{enumerate}

To overcome the above mentioned disadvantages of Adam and AMSGrad, a new optimizer with an adaptive learning parameter, C-Adam, is introduced here. The second moment of C-Adam uses a convex combination of the maximum of the past seen second order moments defined as 
\begin{equation}
    \tilde{v_{t}} = \lambda\tilde{v}_{t-1} + (1-\lambda)\max(\tilde{v}_{t-1}, v_{t}),
\end{equation}
where we choose $\lambda = \dfrac{\tilde{v}_{t-1}}{\max(\tilde{v}_{t-1}, v_{t})}$.

The algorithm for C-Adam is given below (refer Algorithm \ref{customadam algo}).\\

\begin{algorithm}
\caption{C-Adam}
\label{customadam algo}
\begin{algorithmic}[1]

\Require$x_1 \in \mathcal{F}$, step sizes $\{\alpha_t\}_{t=1}^{T}$, 
$\{\beta_{1t}\}_{t=1}^{T}$, $\beta_2$
\State Initialize $m_0 = 0$, $v_0 = 0$, $\tilde{v}_0 = 0$

\For{$t = 1$ to $T$}
    \Statex $g_t = \nabla f_t(x_t)$
    
    \Statex $m_t = \beta_{1t} m_{t-1} + (1 - \beta_{1t}) g_t$

    \Statex $v_t = \beta_2 \tilde{v}_{t-1} + (1 - \beta_2) g_t^2$

    \Statex $\lambda = \tilde{v}_{t-1} / \max(\tilde{v}_{t-1}, v_t)$

    \Statex $\tilde{v}_t = (1-\lambda)\max(\tilde{v}_{t-1}, v_t) + \lambda\tilde{v}_{t-1}$
    
    \Statex $\tilde{V}_t = \operatorname{diag}(\tilde{v}_t)$
    
    \Statex $x_{t+1} = \Pi_{\mathcal{F}, \sqrt{\tilde{V}_t}}
    \left(x_t - \alpha_t \frac{m_t}{\sqrt{\tilde{v}_t+\epsilon}} \right)$
\EndFor{}
\end{algorithmic}
\end{algorithm}

We discuss the reason for the choice of $\lambda$ as above. Our main aim is to introduce a new optimizer that performs better than the existing ones, with a particular focus on Adam and AMSGrad. The case when $v_{t-1} > v_{t}$, Adam need not have a monotonic behavior, meanwhile AMSGrad provides the monotonicity and thus theoretical convergence. We use the same idea and analysis of AMSGrad for C-Adam and define $\tilde{v}_t = v_{t-1}$. On the other hand, when $v_{t-1} \leq v_{t}$, AMSGrad might face the issue of premature vanishing step size. In this case, we move towards a penalized version of AMSGrad and introduce a customized combination of Adam and AMSGrad. To satisfy these requirements we choose, $\lambda = \dfrac{\tilde{v}_{t-1}}{\max(\tilde{v}_{t-1}, v_{t})}$. Note that this leads to our two cases (which shall be used throughout the paper). Case 1 : $\tilde{v}_{t-1} > v_{t}$ in which $\lambda=1$ and $\tilde{v}_t=\tilde{v}_{t-1}$, and Case 2 : $\tilde{v}_{t-1} \leq v_{t}$ leading to $\tilde{v}_t=\lambda \tilde{v}_{t-1}+(1-\lambda)v_t$.

Let us consider the T$^{th}$ iteration, where $\max(\tilde{v}_{t-1}, v_t)$ yields the $v_t$ and $\tilde{v}_{t}$ update using the following update rule $\tilde{v}_{t} = \lambda\tilde{v}_{t-1} + (1-\lambda)v_{t}$, which takes a value somewhere between ${v}_{t}$ and $\tilde{v}_{t-1}$. This allows the adaptive learning rate to be significantly large enough to avoid stagnation and reduces the adverse effect of the maximization constraint of AMSGrad. This change in the algorithm has a drastic effect in the performace and that can be clearly seen in Fig. \ref{noisy_data_result_comparison}. In the case of noisy dataset, the gradients of the Adam and AMSGrad fluctuate due to the rapid changes in the adaptive learning rate that leads to a poor prediction as highlighted in Fig. \ref{noisy_data_result_comparison}.

\begin{figure}[]
    \centering

    \begin{subfigure}[t]{0.4\textwidth}
        \centering
        \begin{overpic}[width=\linewidth]{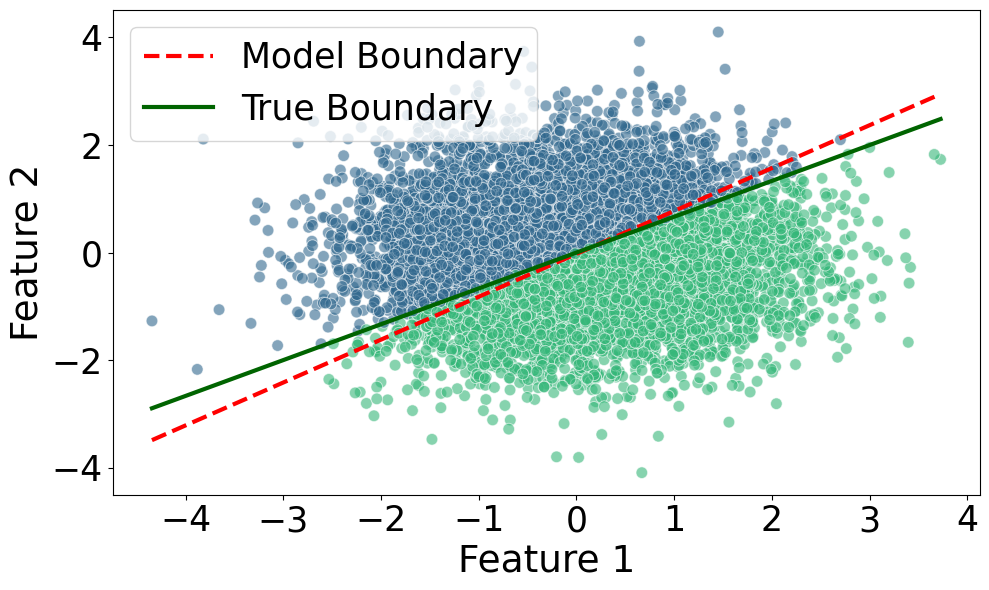}
        \end{overpic}
        \vspace{0.1cm}
        {\small (a) Adam}
    \end{subfigure}
    \hspace{0.1cm}
    \begin{subfigure}[t]{0.4\textwidth}
        \centering
        \begin{overpic}[width=\linewidth]{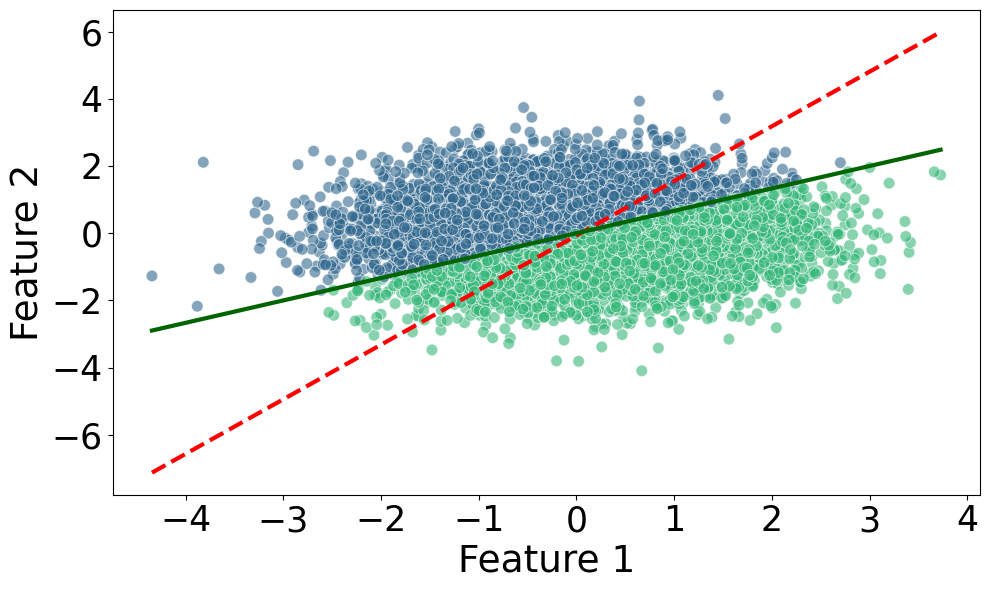}
        \end{overpic}
        \vspace{0.1cm}
        {\small (b) AMSGrad}
    \end{subfigure}
    \vspace{0.3cm}
    \begin{subfigure}[t]{0.4\textwidth}
        \centering
        \begin{overpic}[width=\linewidth]{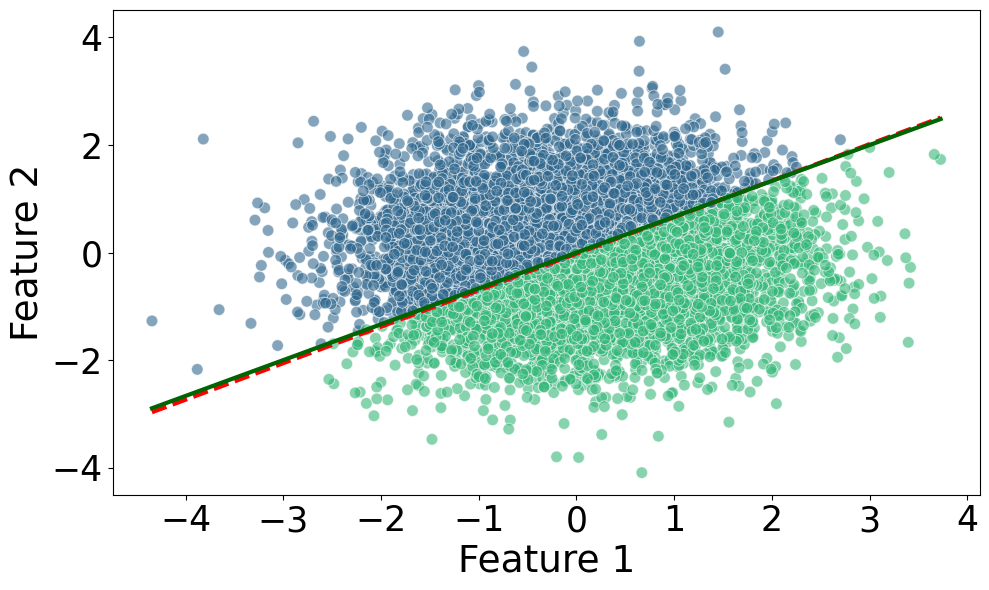}
        \end{overpic}
        \vspace{0.1cm}
        {\small (c) C-Adam}
    \end{subfigure}

    \caption{Data points along with the true boundary and the corresponding linear classifier-based model boundary created by using the optimizer (a) Adam, (b) AMSGrad and (c) C-Adam. The classifier decision boundary is trained on 10000 two-dimensional randomly generated data points, with 10$\%$ noise induced by flipping the labels in both the training and validation sets. Blue and green dots correspond to the two classes used.}
    \label{noisy_data_result_comparison}

\end{figure}

\begin{rem}
The non-negativity of 
    $\displaystyle \psi_{t} = \left( \frac{\sqrt{v_{t}}}{\alpha_{t}} - \frac{\sqrt{v_{t-1}}}{\alpha_{t-1}} \right)$
    is not guaranteed by Adam in an online convex optimization setting. This may lead to non-convergence as discussed in \cite{ReddiKaleKumar2019}. The AMSGrad algorithm worked on this issue and guaranteed $\psi_t \geq 0$ for all $t\in [T]$. We show that C-Adam also guarantees this, thus resolving issues related to the convergence of the model. We prove this in the following lemma. 
\end{rem} 

As mentioned earlier, throughout our proofs we consider two cases: Case 1 when $\tilde{v}_{t-1}$ is the maximum and Case 2 when $v_t$ is the maximum. Also inequalities and equalities applied on vectors are entry wise.
\begin{lem}
    Consider an online convex optimization problem. The optimization rule for C-Adam guarantees that $ \dfrac{\sqrt{\tilde{v}_{t}}}{\alpha_{t}}$ is an increasing function.
\end{lem}
\begin{proof}
    The second moment of C-Adam is $\tilde{v}_{t} = \lambda\tilde{v}_{t-1} + (1-\lambda)\max(\tilde{v}_{t-1}, v_{t})$ with $\lambda = \dfrac{\tilde{v}_{t-1}}{\max(\tilde{v}_{t-1}, v_{t})}$. Without loss of generality, we consider the step size to be a constant, i.e., $\alpha_t = \alpha$.

    \textit{Case 1:} $v_{t} < \tilde{v}_{t-1}$\\
In this case $\lambda = 1$ and $\tilde{v}_{t} = \tilde{v}_{t-1}$ leading to $\psi_t$ becoming 0, thus satisfying the condition.

\textit{Case 2:} $v_{t} \geq \tilde{v}_{t-1}$\\
In this case $\lambda = \dfrac{\tilde{v}_{t-1}}{v_{t}}$ and $\tilde{v}_{t} = \lambda\tilde{v}_{t-1} + (1-\lambda)v_{t}$. We have
\begin{align*}
\psi_{t} &= \dfrac{\sqrt{\tilde{v}_t}}{\alpha} - \dfrac{\sqrt{\tilde{v}_{t-1}}}{\alpha}\\
         &= \dfrac{1}{\alpha}(\sqrt{\lambda\tilde{v}_{t-1} + (1-\lambda)\max(\tilde{v}_{t-1}, v_{t})} -  \sqrt{\tilde{v}_{t-1}})\\ 
         &= \dfrac{1}{\alpha}(\sqrt{\lambda\tilde{v}_{t-1} + (1-\lambda)v_{t}} -  \sqrt{\tilde{v}_{t-1}})\\
        &= \dfrac{1}{\alpha}(\sqrt{\lambda\tilde{v}_{t-1} + \tilde{v}_{t-1} - \tilde{v}_{t-1} + (1-\lambda)v_{t}} -  \sqrt{\tilde{v}_{t-1}})\\
        &= \dfrac{1}{\alpha}(\sqrt{\tilde{v}_{t-1} + (1 - \lambda)(v_{t} - \tilde{v}_{t-1})} - \sqrt{\tilde{v}_{t-1}})\\
        &\geq 0 \ (\text{since}\ (1-\lambda)\geq 0\ \text{and}\ (v_t-\tilde{v}_{t-1})\geq 0)
\end{align*}
In both the cases $\psi \geq 0$, which yields the desired result.
\end{proof}

\begin{lem}\label{mt in terms of vt}
Let $m_{t}$ and $v_{t}$ be the first and second moments defined in Algorithm \ref{adam algo} at time t. Choose $\beta_1$ and $\beta_2$ such that $\delta=\dfrac{\beta_{1}}{\sqrt{\beta_{2}}} < 1$. Then 
$$m_{t} \leq \dfrac{(1-\beta_{1})}{\sqrt{1-\beta_{2}}}\sqrt{\dfrac{v_t}{1-\delta^2}}.$$ 
\begin{proof} We have
        $$v_{t} = (1-\beta_{2})\sum_{j=1}^t{\beta_{2}^{t-j}g_j^{2}},\ \ 
        m_{t} = (1 - \beta_{1})\sum_{i=1}^{t}{\beta_{1}^{t-i}g_{i}}$$
 Applying Cauchy-Schwarz inequality, we get
    \begin{align*}
        m_{t} &\leq  (1-\beta_{1})\left(\sum_{i=1}^{t}{g_{i}^{2}}\right)^{1/2}\left(\sum_{i=1}^{t}       {(\beta_{1}^{t-i})^{2}}\right)^{1/2} \\
        &\leq  (1-\beta_{1})\left(\sum_{i=1}^{t}{\beta_{2}^{t-i}g_{i}^{2}}\right)^{1/2}\left(\sum_{i=1}^{t}{\left(\dfrac{\beta_{1}^{2}}{\beta_{2}}\right)^{t-i}}\right)^{1/2} \\
        &\leq \dfrac{(1-\beta_{1})}{\sqrt{1-\beta_{2}}}\left((1-\beta_{2})\left(\sum_{i=1}^{t}{\beta_{2}^{t-i}g_{i}^{2}}\right)\right)^{1/2}\left(\sum_{i=1}^{t}{\left(\dfrac{\beta_{1}^{2}}{\beta_{2}}\right)^{t-i}}\right)^{1/2} \\
        &\leq \dfrac{(1-\beta_{1})}{\sqrt{1-\beta_{2}}}\sqrt{v_{t}}\left(\sum_{i=1}^{t}{\left(\dfrac{\beta_{1}^{2}}{\beta_{2}}\right)^{t-i}}\right)^{1/2} \\
        &\leq \dfrac{(1-\beta_{1})}{\sqrt{1-\beta_{2}}}\sqrt{v_{t}}\sqrt{\dfrac{\beta_{2}}{\beta_{2} - \beta_{1}^{2}}}
    \end{align*}
   The proof is completed by putting $\dfrac{\beta_1}{\sqrt{\beta_2}}=\delta$.
\end{proof}
\end{lem}

\begin{lem}\label{proof for bound}
Let us consider C-Adam and consider the parameter assumptions as defined in Lemma \ref{mt in terms of vt}. Let $\zeta = \dfrac{\alpha \sqrt{(1+\log T)}}{(1-\delta)\sqrt{1-\beta_2}}\displaystyle\sum_{i=1}^d{\lVert g_{1:T, i}\rVert_{2}}$. Then we have the following bound:
 $$\sum_{t=1}^T {\alpha_{t}\lVert \tilde{V}_{t}^{-1/4}m_{t} \rVert^2} \leq \zeta \max( K_1, K_2),$$
 where $K_1 = \dfrac{1}{(1-\beta_{1})}$ and  $K_2 = \dfrac{(1-\beta_{1})^{2}\lambda_{\infty}}{(1-\sqrt{\beta_{2}})(1+\delta)}$ with $\lambda_\infty=\displaystyle\max_i\frac{1}{\sqrt{1-\lambda_i}}$.

\begin{proof}
We discuss the two cases that arise in the updating step of the optimizer C-Adam.

\textit{Case 1:} Suppose $\max(v_{t}, \tilde{v}_{t-1}) = \hat{v}_{t-1}$. Then $\tilde{v}_{t} \geq v_{t}$ and from \cite{ReddiKaleKumar2019}, we have
\begin{align*}
    \sum_{t=1}^T{\alpha_{t}\lVert \tilde{V}_{t}^{-1/4}m_{t} \rVert^2} &\leq \dfrac{\alpha\sqrt{(1+\log{T})\beta_{2}}}{(1-\beta_{1})(\sqrt{\beta_2}-\beta_{1})\sqrt{1-\beta_{2}}}\sum_{i=1}^d{\lVert g_{1:T, i}\rVert_{2}}=\zeta K_1
\end{align*}
\textit{Case 2: } Suppose $\max(v_{t}, \tilde{v}_{t-1}) = v_{t}$, using the nonnegativity of $\tilde{v}_{t-1}$, we have
$$\tilde{v}_{t} \geq (1-\lambda)v_{t} + \lambda\tilde{v}_{t-1}\geq (1-\lambda)v_{t}$$
    
Now we try to find an upper bound for the above quantity.
\begin{align*}
    \sum_{t=1}^T{\alpha_{t}\lVert \tilde{V}_{t}^{-1/4}m_{t} \rVert^2} &= \sum_{t=1}^{T-1}{\alpha_{t}\lVert \tilde{V}_{t}^{-1/4}m_{t} \rVert^2} + \alpha_{T}\sum_{i=1}^d \dfrac{m_{T,i}^2}{\sqrt{\tilde{v}_{T,i}}} \\
    &\leq \sum_{t=1}^{T-1}{\alpha_{t}\lVert \tilde{V}_{t}^{-1/4}m_{t} \rVert^2} + \alpha_{T}\sum_{i=1}^d \dfrac{m_{T,i}^2}{\sqrt{(1-\lambda_{i})v_{T,i}}}
\end{align*}
Using Lemma \ref{mt in terms of vt} we have $m_{t} \leq  \dfrac{(1-\beta_{1})}{\sqrt{1-\beta_{2}}}\sqrt{\dfrac{v_{t}}{1-\delta^2}}$ in the second term of the above inequality, we have

\begin{align*}
    \alpha_{T}\sum_{i=1}^d 
     \dfrac{m_{T,i}^2}{\sqrt{(1-\lambda_{i})v_{T,i}}} 
     &\leq \alpha_{T}\sum_{i=1}^d \dfrac{(1-\beta_{1})^{2}v_{T,i}}{(1-\beta_{2})(1-\delta^2)\sqrt{(1-\lambda_{i})v_{T,i}}}\\
     &\leq \alpha_{T}\sum_{i=1}^{d}\dfrac{(1-\beta_{1})^{2}\lambda_{\infty}\sqrt{(1-\beta_{2})\sum_{j=1}^T{\beta_{2}^{T-j}g_{j,i}^{2}}}}{(1-\beta_{2})(1-\delta^2)} \\
     &\leq \alpha\sum_{i=1}^{d}{\dfrac{(1-\beta_{1})^{2}\lambda_{\infty}}{\sqrt{1-\beta_2}(1-\delta^2)}\sum_{j=1}^{T}{\sqrt{\dfrac{\beta_{2}^{T-j}g_{j,i}^{2}}{T}}}}
\end{align*}
Now the above equation can be written as 
\begin{align*}
    \sum_{t=1}^T{\alpha_{t}\lVert \tilde{V}_{t}^{-1/4}m_{t} \rVert^2} &\leq \sum_{t=1}^{T-1}{\alpha_{t}\lVert \tilde{V}_{t}^{-1/4}m_{t} \rVert^2} + \alpha\sum_{i=1}^{d}{\dfrac{(1-\beta_{1})^{2}\lambda_{\infty}}{\sqrt{1-\beta_2}(1-\delta^2)}\sum_{j=1}^{T}{\sqrt{\dfrac{\beta_{2}^{T-j}g_{j,i}^{2}}{T}}}}\\
    &=  \alpha\sum_{t=1}^{T}{\sum_{i=1}^{d}{\dfrac{(1-\beta_{1})^{2}\lambda_{\infty}}{\sqrt{1-\beta_2}(1-\delta^2)}\sum_{j=1}^{t}{\sqrt{\dfrac{\beta_{2}^{t-j}g_{j,i}^{2}}{t}}}}}\\
    &= \alpha \dfrac{(1-\beta_{1})^{2}\lambda_{\infty}}{\sqrt{1-\beta_2}(1-\delta^2)}\sum_{i=1}^{d}{\sum_{t=1}^{T}\dfrac{1}{\sqrt{t}}{\sum_{j=1}^{t}{\sqrt{\beta_2^{t-j}g_{j,i}^{2}}}}} \\
     &\leq \alpha \dfrac{(1-\beta_{1})^{2}\lambda_{\infty}}{\sqrt{1-\beta_2}(1-\delta^2)}\sum_{i=1}^{d}{\sum_{t=1}^{T}{\dfrac{|g_{t,i}|}{\sqrt{t}(1-\sqrt{\beta_{2}})}}}\\
    &\leq \alpha \dfrac{(1-\beta_{1})^{2}\lambda_{\infty}}{\sqrt{1-\beta_2}(1-\sqrt{\beta_{2}})(1-\delta^2)}\sum_{i=1}^{d}{\sum_{t=1}^{T}{\dfrac{|g_{t,i}|}{\sqrt{t}}}}\\
    &\leq \alpha \dfrac{(1-\beta_{1})^{2}\lambda_{\infty}}{\sqrt{1-\beta_2}(1-\sqrt{\beta_{2}})(1-\delta^2)}\sum_{i=1}^{d}{(\sum_{t=1}^{T}{\dfrac{1}{t}})^{1/2}(\sum_{t=1}^{T}{|g_{t,i}|^{2}})^{1/2}}\\
    &\leq \alpha \dfrac{(1-\beta_{1})^{2}\lambda_{\infty}\sqrt{1+\log{T}}}{\sqrt{1-\beta_2}(1-\sqrt{\beta_{2}})(1-\delta^2)}\sum_{i=1}^{d}{\lVert g_{1:T,i}\rVert_{2}}=\zeta K_2
\end{align*}
\end{proof}
\end{lem}

\begin{thm}\label{custom adam bound}
Let $\{x_t\}, \{m_t\}$, and $\{\tilde{v}_t\}$ be the sequences obtained from the C-Adam \ref{customadam algo}. Choose $\alpha$ and let $\alpha_t=\alpha/\sqrt{t}$, and choose $\beta_1$ and $\beta_2$ such that $\beta_{1t} \leq \beta_1, \; t \in [T]$, and $\frac{\beta_1}{\sqrt{\beta_2}} \leq 1$. Assume that the domain $(\Omega)$ has a bounded diameter $\Omega_{\infty}$ and the convex function $f$ has bounded gradients $\lVert \nabla f_t (x) \rVert \leq M$ for all $t \in [T]$ and $x \in \Omega$. Then the C-Adam has following regret bound:

 $$\sum_{t=1}^T f_{t}(x_{t}) - f_{t}(x^{*}) \leq \dfrac{\Omega_{\infty}^{2}}{2\alpha_{T}(1-\beta_{1})}\sum_{i=1}^{d}{\tilde{v}_{T,i}^{1/2}} + \dfrac{\Omega_{\infty}^{2}}{(1-\beta_{1})^{2}}\sum_{t=1}^{T}{\sum_{i=1}^{d}{\dfrac{\beta_{1t}\tilde{v}_{t,i}^{1/2}}{\alpha_{t}}}} +\zeta \max(K_1, K_2)$$
 where $\zeta$, $K_1$, $K_2$ and $\lambda_\infty$ are as mentioned in Lemma \ref{proof for bound}.

\begin{proof}
    Using Taylor series we know that,
\begin{align*}
f(y) = f(x) + \nabla f(x)(y - x) + \dfrac{1}{2}(y-x)^{T}\nabla^{2}f(x)(y-x)
\end{align*}
Since $f$ is convex, the Hessian of $f$ is positive semi definite, that is, $\dfrac{1}{2}(y-x)^{T}\nabla^{2}f(x)(y-x) \geq  0$.

Thus we have
\begin{align*}
    f(x) - f(y) &\leq \nabla f(x)(y - x) \\
    f(x) - f(x^*) &\leq \nabla f(x)(x^* - x),\ \text{for } y=x^*
\end{align*}
Using the definition of projection,
\begin{align*}
    x_{t+1} &= \Pi_{\mathcal{F}, \tilde{V}_t^{1/2}}
    \left(x_t - \alpha_t\tilde{V}_t^{-1/2}m_t \right)\\
    &= \min_{x \in \mathcal{F}}\lVert \tilde{V}_t^{1/4}(x - (x_t - \alpha_t\tilde{V}_t^{-1/2}m_t)) \rVert
\end{align*}
Then we have,
\begin{align*}
    \lVert \tilde{V}_{t}^{1/4}(x_{t+1} - x^{*})\rVert^{2} 
    &= \lVert \tilde{V}_t^{1/4}(x_t  - \alpha_t\tilde{V}_t^{-1/2}m_t - x^{*}) \rVert^2 \\
    &= \lVert \tilde{V}_{t}^{1/4}(x_{t} - x^{*})\rVert^2 + \alpha_t^{2} \lVert\tilde{V}_{t}^{-1/4}m_{t}\rVert^2  
     -2\alpha_{t}\langle \beta_{1t}m_{t-1} +
    (1-\beta_{1t})g_{t}, (x_{t} - x^{*})\rangle 
\end{align*}
Upon rearranging the above equation, we get
\begin{align*}
    \langle g_{t}, (x_{t} - x^{*}) \rangle &= \dfrac{1}{2\alpha_{t}(1-\beta_{1t})}\left[\lVert \tilde{V}_{t}^{1/4}(x_{t} - x^{*})\rVert^2 - \lVert \tilde{V}_{t}^{1/4}(x_{t+1} - x^{*})\rVert^2\right] \\ & \hspace{1cm}+\dfrac{\alpha_{t}}{2(1-\beta_{1t})}\lVert\tilde{V}_{t}^{-1/4}m_{t}\rVert^2 + \dfrac{\beta_{1t}}{(1-\beta_{1t})}\langle m_{t-1}, (x_{t} - x^{*}) \rangle
    \end{align*}
    Using Young's inequality in the last term, we have
    \begin{align*}
\langle m_{t-1}, (x_{t} - x^{*}) \rangle &= \langle \alpha_t \tilde{V}_t^{-1/4} m_{t-1}, \dfrac{1}{\alpha_t}\tilde{V}_t^{1/4}(x_{t} - x^{*}) \rangle \\
    &\leq \dfrac{\alpha_{t}}{2}\lVert\tilde{V}_{t}^{-1/4}m_{t}\rVert^2 
+\dfrac{1}{2\alpha_t}\lVert \tilde{V}_{t}^{1/4}(x_{t} - x^{*})\rVert^2 
\end{align*}
Combining the above and using the fact that $\beta_1\geq \beta_{1t}$ for all $t$ for terms in the denominator, we get
\begin{align*}
\langle g_t, (x_t-x^*)\rangle &\leq  \dfrac{1}{2\alpha_{t}(1-\beta_{1})}\left[\lVert \tilde{V}_{t}^{1/4}(x_{t} - x^{*})\rVert^2 - \lVert \tilde{V}_{t}^{1/4}(x_{t+1} - x^{*})\rVert^2\right] \\
    &\hspace{1cm}+\dfrac{\alpha_{t}}{(1-\beta_{1})}\lVert\tilde{V}_{t}^{-1/4}m_{t}\rVert^2 + \dfrac{\beta_{1t}}{2\alpha_{t}(1-\beta_{1})}\lVert \tilde{V}_{t}^{1/4}(x_{t} - x^{*})\rVert^2
\end{align*}
Now we know that $f_{t}(x_{t}) - f_{t}(x^{*}) \leq \langle g_t, (x_{t} - x^{*})\rangle$.\\

Using the bounds obtained in Lemma \ref{proof for bound} and \cite[Theorem 4]{ReddiKaleKumar2019}, we have 
\begin{align*}
    \sum_{t=1}^T f_{t}(x_{t}) - f_{t}(x^{*}) &\leq \dfrac{1}{2\alpha_{1}(1-\beta_{1})}\sum_{i=1}^{d}{\tilde{v}_{1,i}^{1/2}(x_{1,i} - x_{i}^{*})^{2}}  +\dfrac{1}{2(1-\beta_{1})}\sum_{t=2}^{T}{\sum_{i=1}^{d}{(x_{t,i} - x_{i}^{*})^{2}\left[\dfrac{\tilde{v}_{t,i}^{1/2}}{\alpha_{t}} - \dfrac{\tilde{v}_{t-1,i}^{1/2}}{\alpha_{t-1}}\right]}}  \\
    &\hspace{1cm}+\dfrac{1}{(1-\beta_{1})^{2}}\sum_{t=1}^{T}\sum_{i=1}^{d}\dfrac{\beta_{1t}(x_{t,i} - x_{i}^{*})^{2}}{\alpha_t}\tilde{v}_{t,i}^{1/2} + \zeta \max(K_1, K_2)\\
    &\leq \dfrac{1}{2\alpha_{T}(1-\beta_{1})}\sum_{i=1}^{d}{\Omega_{\infty}^{2}\tilde{v}_{T,i}^{1/2}} + 
    \dfrac{1}{2(1-\beta_{1})}\sum_{t=2}^{T}{\sum_{i=1}^{d}{\Omega_{\infty}^{2}\left[\dfrac{\tilde{v}_{t,i}^{1/2}}{\alpha_{t}} - \dfrac{\tilde{v}_{t-1,i}^{1/2}}{\alpha_{t-1}}\right]}} \\
    &\hspace{1cm}+\dfrac{1}{(1-\beta_{1})^{2}}\sum_{t=1}^{T}{\sum_{i=1}^{d}{\dfrac{\Omega_{\infty}^{2}\beta_{1t}}{\alpha_{t}}}}\tilde{v}_{t,i}^{1/2} + \zeta \max(K_1, K_2) \\
    &\leq \dfrac{\Omega_{\infty}^{2}}{2\alpha_{T}(1-\beta_{1})}\sum_{i=1}^{d}{\tilde{v}_{T,i}^{1/2}} + \dfrac{\Omega_{\infty}^{2}}{(1-\beta_{1})^{2}}\sum_{t=1}^{T}{\sum_{i=1}^{d}{\dfrac{\beta_{1t}}{\alpha_{t}}}}\tilde{v}_{t,i}^{1/2} +\zeta \max(K_1, K_2)
\end{align*}
Thus the regret bound for C-Adam for the above mentioned conditions is derived to be,
\begin{align*}
    \sum_{t=1}^T f_{t}(x_{t}) - f_{t}(x^{*}) &\leq \dfrac{\Omega_{\infty}^{2}}{2\alpha_{T}(1-\beta_{1})}\sum_{i=1}^{d}{\tilde{v}_{T,i}^{1/2}} + \dfrac{\Omega_{\infty}^{2}}{(1-\beta_{1})^{2}}\sum_{t=1}^{T}{\sum_{i=1}^{d}{\dfrac{\beta_{1t}\tilde{v}_{t,i}^{1/2}}{\alpha_{t}}}} +\zeta \max(K_1, K_2)
\end{align*}
where $\zeta$, $K_1$ and $K_2$ are as mentioned in Lemma \ref{proof for bound}.
\end{proof}
\end{thm}

\begin{rem}
    In the above theorem and lemmas, we are assuming that throughout iterations 1 to T, either case 1 or case 2 applies. Allowing a mix of these cases during iterations will render the proof tedious. Hence, for the sake of simplicity, we assume that either of the cases appear across all iterations, even if that may not be the case practically.
\end{rem}

\section{Experimental Validation of C-Adam}\label{exp}

In this section, the proposed optimizer is empirically evaluated on both synthetic and real-world datasets. Specifically, multiclass classification is studied using logistic regression, a fully connected neural network, and a convolutional neural network, with three optimizers. The experimental settings used are as shown in Table \ref{tldr hyperparameter}. Also, irrespective of the experiments and optimizers, the cross-entropy-based loss function is used throughout this section.

\begin{table}[]
\centering
\caption{Hyperparameter settings used across experiments.}
\label{tldr hyperparameter}
\small
\begin{tabular}{p{3.2cm}cccccccc}
\toprule
\textbf{Experiment} & $\alpha_0$ & $\beta_1$ & $\beta_2$ & $\epsilon$ & Batch & Iter. & Reg. \\
\midrule
Synthetic Experiment & 0.5 & 0.9 & 0.99 & $10^{-8}$  & 1 & $5\times 10^6$ & None \\
Logistic Regression & $10^{-3}$ & 0.9 & 0.999 & $10^{-12}$ & 128 & 200 & $10^{-4}$ \\
Fully Connected Neural Network & $10^{-3}$ & 0.9 & 0.999 & $10^{-12}$ & 128 & 500 & $10^{-4}$ \\
Convolutional Neural Network & $10^{-3}$ & 0.9 & 0.999 & $10^{-8}$ & 32 & 3500 & None \\
\bottomrule
\end{tabular}
\end{table}

\subsection{Synthetic Experiment}\label{syn exp section}

Here, the three optimisers are compared on the following problem, and the corresponding results are shown in Fig.~\ref{fig:adam_vs_amsgrad_vs_customadam}:

\begin{equation}\label{synthetic exp}
f_t(x) =
\begin{cases}
1010x, & \text{with } p=0.01, \\
-10x, & \text{otherwise}.
\end{cases}
\end{equation}

Here $\epsilon$ is set to $10^{-8}$ to ensure numerical stability and to avoid division by zero in the adaptive update. Under this setting, rather than converging to the optimal point $x=-1$ as reported in \cite{ReddiKaleKumar2019}, Adam is found to converge to the opposite extreme of the feasible set, that is $x=1$, which is a highly suboptimal solution (Fig.~\ref{fig:adam_vs_amsgrad_vs_customadam}a). Adam also exhibits a non-zero asymptotic average regret ($R_t/t$), indicating that the updates remain systematically suboptimal over time (Fig.~\ref{fig:adam_vs_amsgrad_vs_customadam}b). In contrast, both AMSGrad and C-Adam converge to the optimal solution, with C-Adam reaching it in fewer epochs than AMSGrad (Fig.~\ref{fig:adam_vs_amsgrad_vs_customadam}a). This faster convergence may be attributed to the less conservative adaptive update of C-Adam, which allows more effective early-stage progress toward the optimum. This trend is further reflected in Fig.~\ref{fig:adam_vs_amsgrad_vs_customadam}(b), where C-Adam exhibits lower regret than AMSGrad. The non-smoothness of the graph in Fig.~\ref{fig:adam_vs_amsgrad_vs_customadam}(a) is due to the stochastic nature of the problem. For the deterministic variant of the problem \ref{synthetic exp}, the results are almost comparable to the AMSGrad, and the difference in the optimal points gained from both the optimizers is in the order of $10^{-6}$.

\begin{figure}[]
    \centering

    \begin{subfigure}[t]{0.48\textwidth}
        \centering
        \begin{overpic}[width=\linewidth]{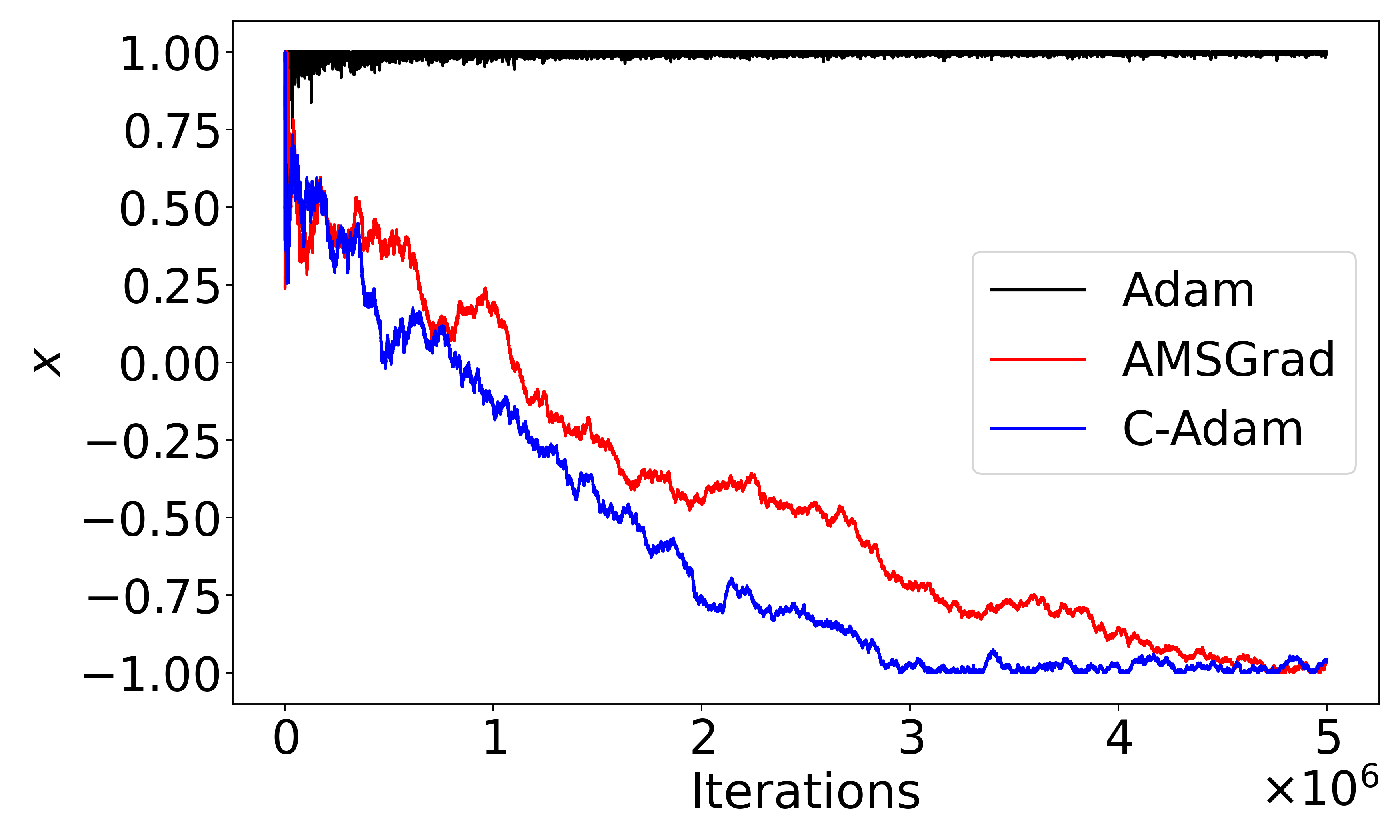}
        \end{overpic}
        \vspace{0.1cm}
        {\small (a) x versus iterations}
    \end{subfigure}
    \hspace{0.1cm}
    \begin{subfigure}[t]{0.48\textwidth}
        \centering
        \begin{overpic}[width=\linewidth]{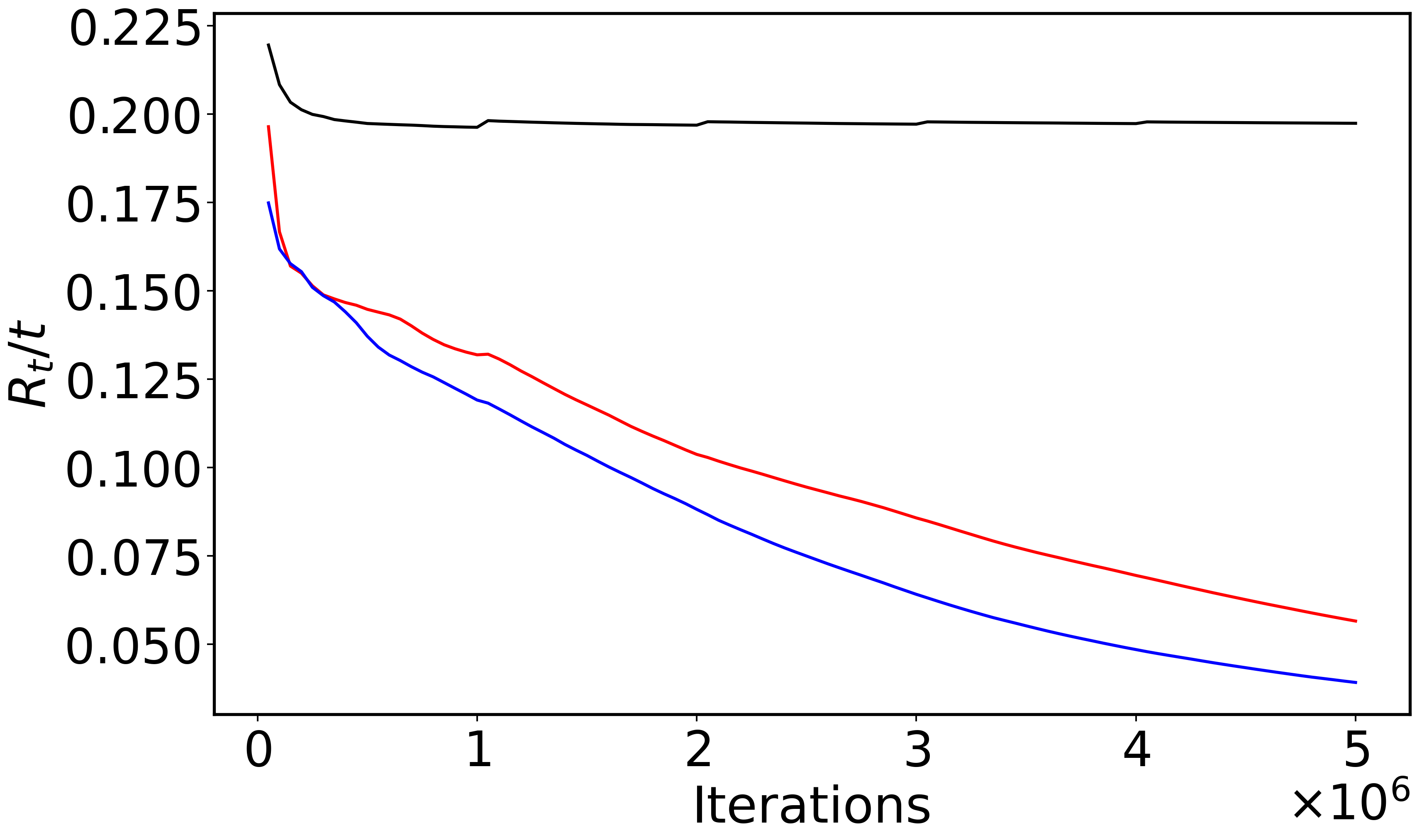}
        \end{overpic}
        \vspace{0.1cm}
        {\small (b) Regret bound versus iteration}
    \end{subfigure}

    \caption{Comparison of Adam, AMSGrad and C-Adam for the synthetic problem \ref{syn exp section}}
  \label{fig:adam_vs_amsgrad_vs_customadam}

\end{figure}

\subsection{Logistic Regression}\label{log section}
In this section, the convergence behaviour of the three optimizers in a convex optimization setting \cite{BoydVandenberghe2004} is investigated by performing multiclass logistic regression over the MNIST dataset. Specifically, the performance of the optimizers is evaluated by their ability to classify digits using 784-dimensional feature vectors obtained by flattening 28 $\times$ 28 grayscale images. The experimental setup described in Table~\ref{tldr hyperparameter} is used, with $\alpha_0 = 0.001$ and $\epsilon = 10^{-12}$. The value of hyperparameters is chosen in such a way that at least one of the optimizers shows sufficient convergence and no overfitting. A batch size of 128 is used, and training is limited to 200 iterations, as the loss does not improve significantly beyond this point.

In Fig. \ref{fig:adam_vs_amsgrad_vs_customadam_logistic}, the convergence characteristics of the three optimizers are shown. C-Adam exhibits a faster reduction in the loss and converges to a lower final value in both training (Fig. \ref{fig:adam_vs_amsgrad_vs_customadam_logistic}a) and validation (Fig. \ref{fig:adam_vs_amsgrad_vs_customadam_logistic}b) compared with Adam and AMSGrad, indicating more efficient optimization under the same training setup. The smooth evolution of the curves suggests stable convergence behaviour under the chosen training setup.

\begin{figure}[]
    \centering

    \begin{subfigure}[t]{0.48\textwidth}
        \centering
        \begin{overpic}[width=\linewidth]{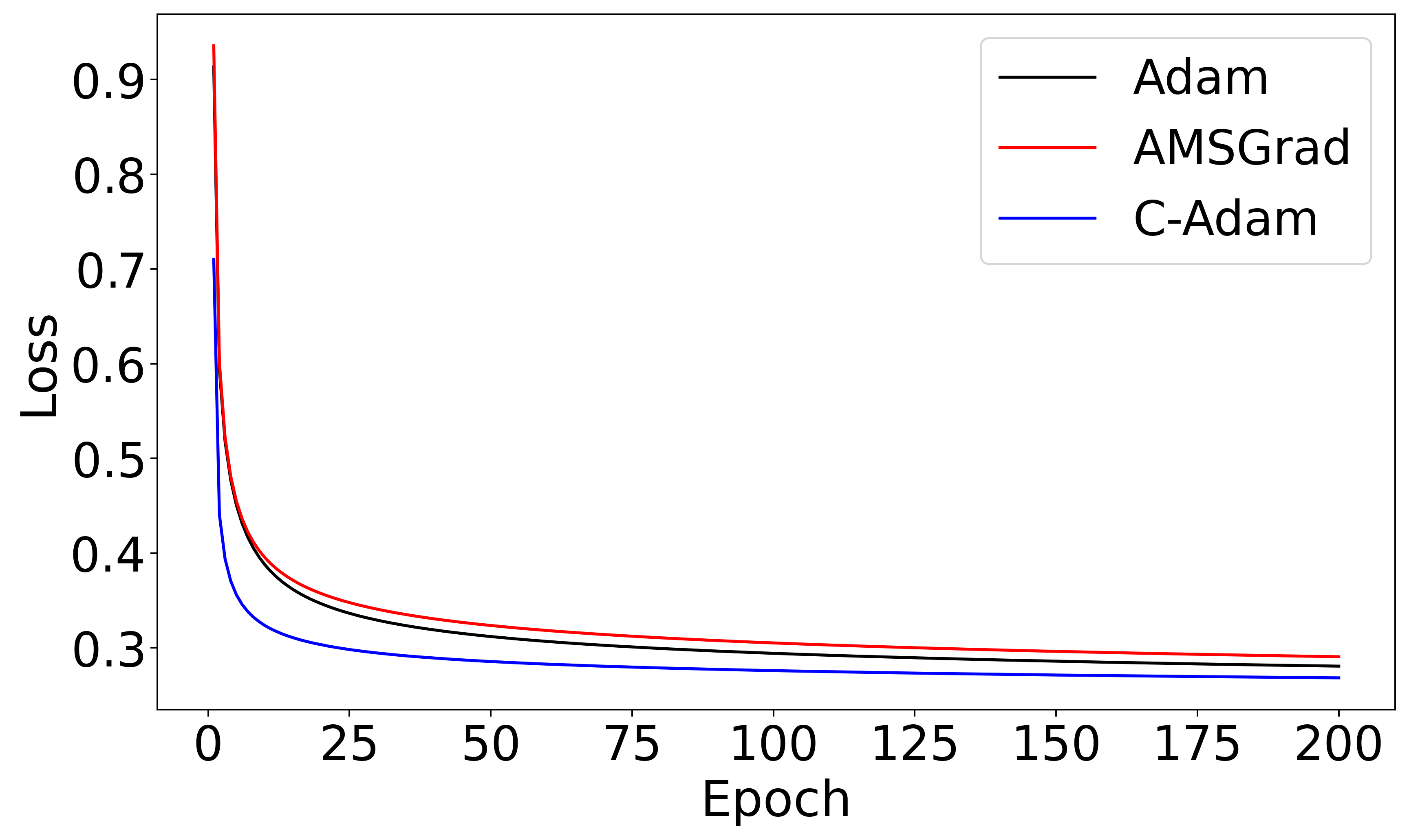}
        \end{overpic}
        \vspace{0.1cm}
        {\small (a) Training}
    \end{subfigure}
    \hspace{0.1cm}
    \begin{subfigure}[t]{0.48\textwidth}
        \centering
        \begin{overpic}[width=\linewidth]{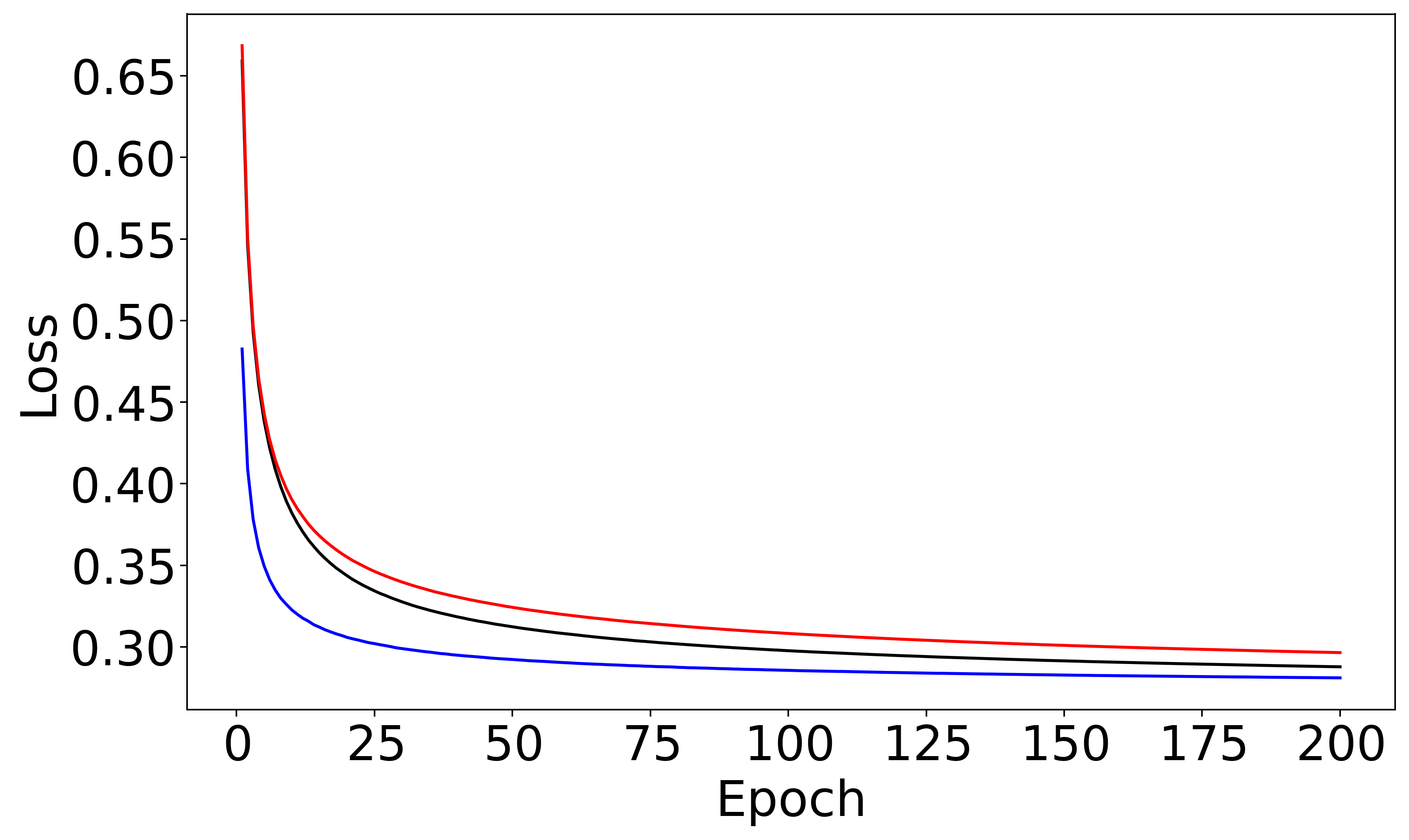}
        \end{overpic}
        \vspace{0.1cm}
        {\small (b) Validation}
    \end{subfigure}

    \caption{
    Training and validation losses of the three optimizers for multiclass classification using logistic regression \ref{log section}}
  \label{fig:adam_vs_amsgrad_vs_customadam_logistic}

\end{figure}

\subsection{Fully Connected Neural Network}\label{fully connected section}
In this section, the proposed optimizer is evaluated on a fully connected neural network with a single hidden layer. A total of 100 neurons are used in the hidden layer, with ReLU as the activation function. The remaining experimental parameters are the same as those given in Table~\ref{tldr hyperparameter}, with the initial learning rate set to $\alpha_0 = 0.001$, $\beta_{1t}$ chosen as a strictly decreasing sequence, and $\beta_2$ kept fixed. To reduce the effective learning rate over iterations, C-Adam uses a convex combination of $\max(\hat{v}_{t-1}, v_{t})$ and $\hat{v}_{t-1}$. Using a diminishing sequence of $\alpha_t$ and $\beta_{1t}$ simultaneously penalizes the learning rate multiple times, which can lead to slightly worse outcomes in some cases, and hence caution is exercised when selecting hyperparameters in such cases. 

The network is trained for multiclass classification on the MNIST dataset. The corresponding performance comparison across the three optimizers is shown in Fig.~\ref{fig:adam_vs_amsgrad_vs_customadam_singlenn}. While Adam and AMSGrad demonstrate a remarkable decrease in training loss (Fig.~\ref{fig:adam_vs_amsgrad_vs_customadam_singlenn}a), C-Adam showed faster convergence compared to the other two optimizers.

\begin{figure}[]
    \centering

    \begin{subfigure}[t]{0.48\textwidth}
        \centering
        \begin{overpic}[width=\linewidth]{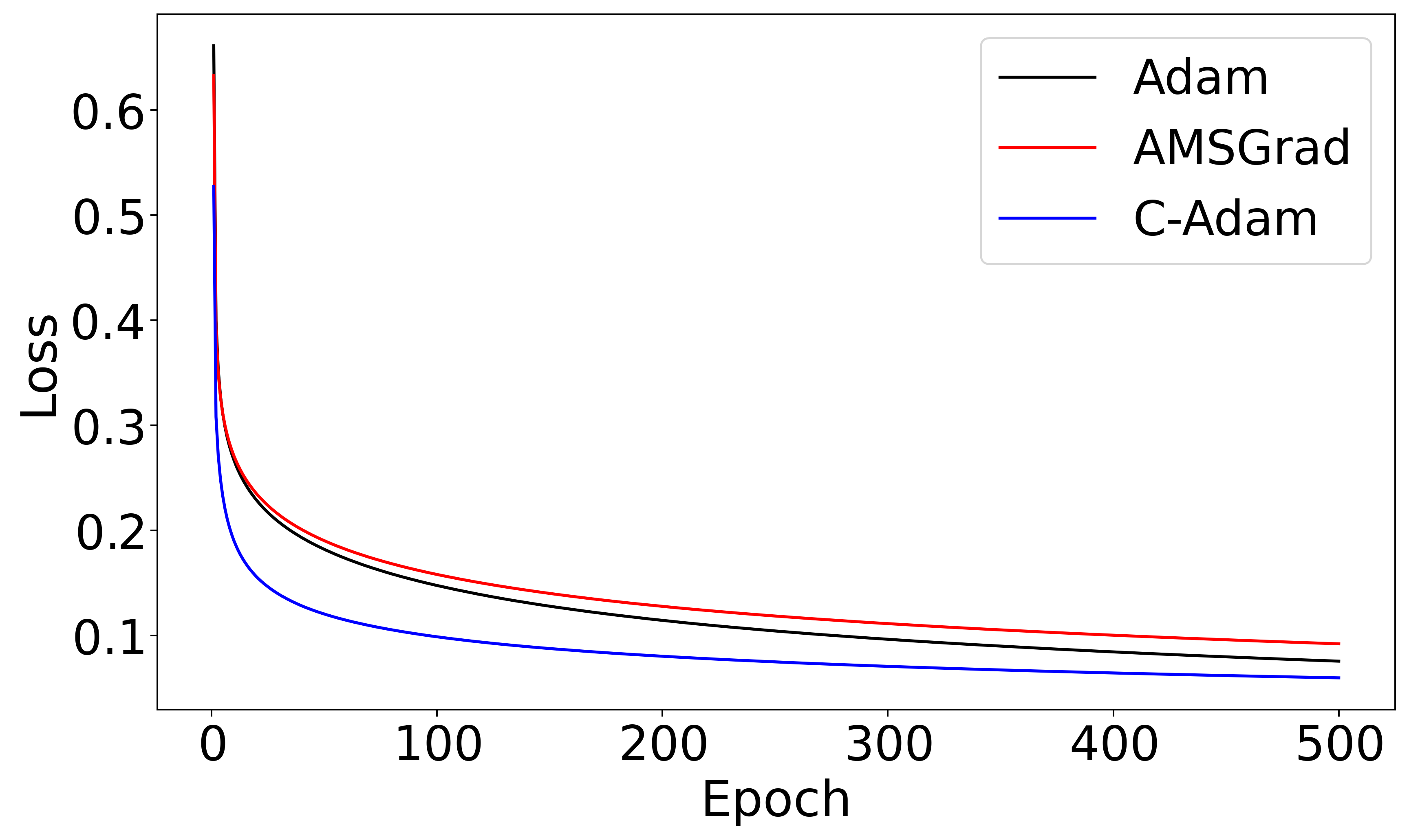}
        \end{overpic}
        \vspace{0.1cm}
        {\small (a) Training}
    \end{subfigure}
    \hspace{0.1cm}
    \begin{subfigure}[t]{0.48\textwidth}
        \centering
        \begin{overpic}[width=\linewidth]{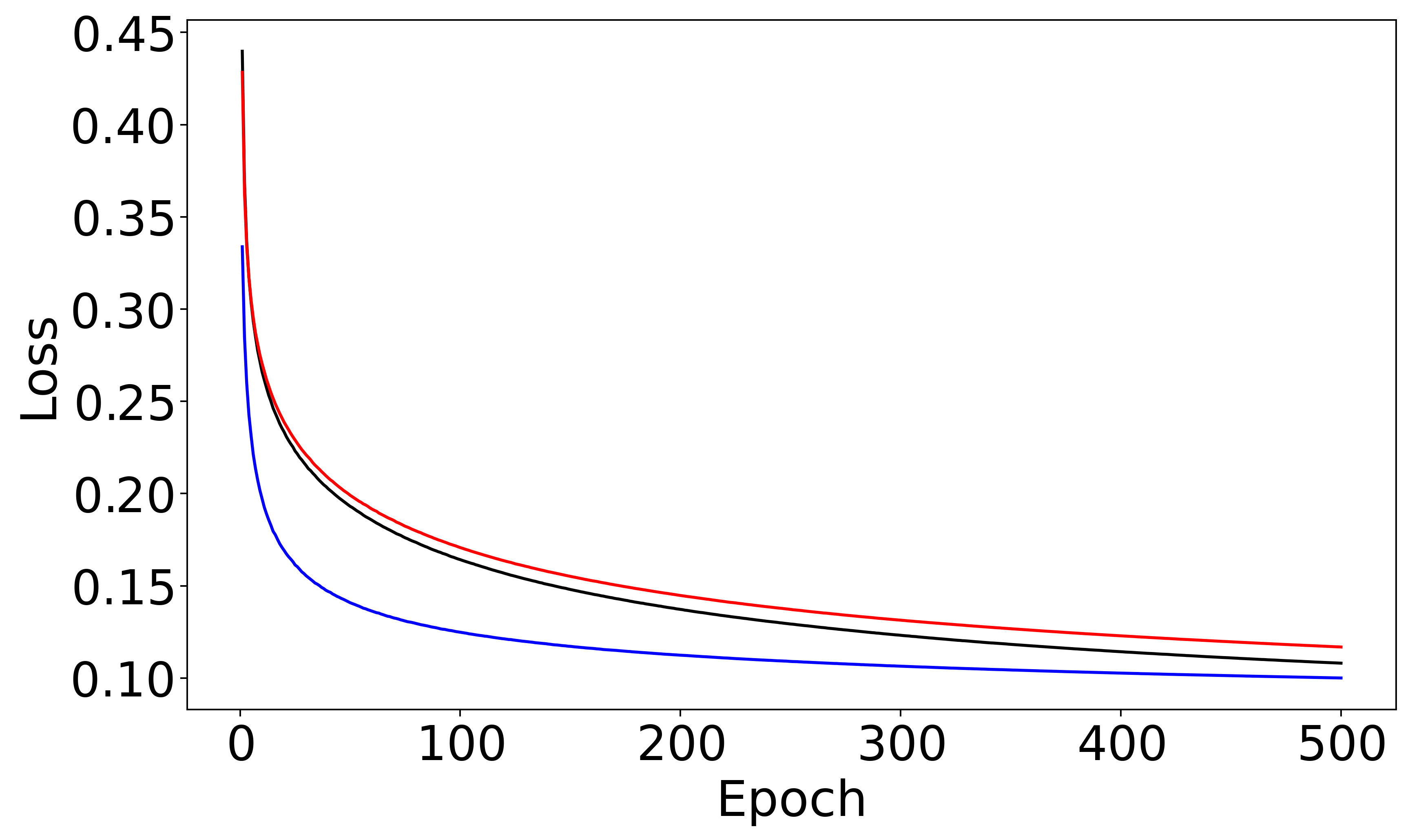}
        \end{overpic}
        \vspace{0.1cm}
        {\small (b) Validation}
    \end{subfigure}

    \caption{Training and validation losses of the three optimizers for multiclass classification using a single hidden layer fully connected neural network \ref{fully connected section}}
\label{fig:adam_vs_amsgrad_vs_customadam_singlenn}

\end{figure}

\subsection{Convolution Neural Network}\label{cnn section}
Here, the proposed optimizer is further evaluated over an image-based multiclass classification using a Convolutional Neural Network (CNN). The CNN model architecture consists of two convolutional layers with 64 filters of size $6\times 6$, each followed by RELU activation, local response normalization, and $2\times 2$ max-pooling for spatial down-sampling. The resulting feature maps of size $64\times 4\times 4$ are flattened into a 1024-dimensional vector and passed through two fully connected layers with 384 and 192 hidden units with a dropout layer in between as used in \cite{ReddiKaleKumar2019}. The experimental setup is the same as explained in Table \ref{tldr hyperparameter} with $\alpha_0 = 0.001$. The model is trained for 3500 iterations with an 80-20 train-test split using the CIFAR-10 dataset, which comprises RGB images of size $32\times 32$ with 3 colour channels for 10 classes. The corresponding performance plots are shown in Fig. \ref{fig:adam_vs_amsgrad_vs_customadam_cnn}. C-Adam exhibits better and faster convergence than other optimizers during training (Fig. \ref{fig:adam_vs_amsgrad_vs_customadam_cnn}a) and validation (Fig. \ref{fig:adam_vs_amsgrad_vs_customadam_cnn}b).

\begin{figure}[]
    \centering

    \begin{subfigure}[t]{0.48\textwidth}
        \centering
        \begin{overpic}[width=\linewidth]{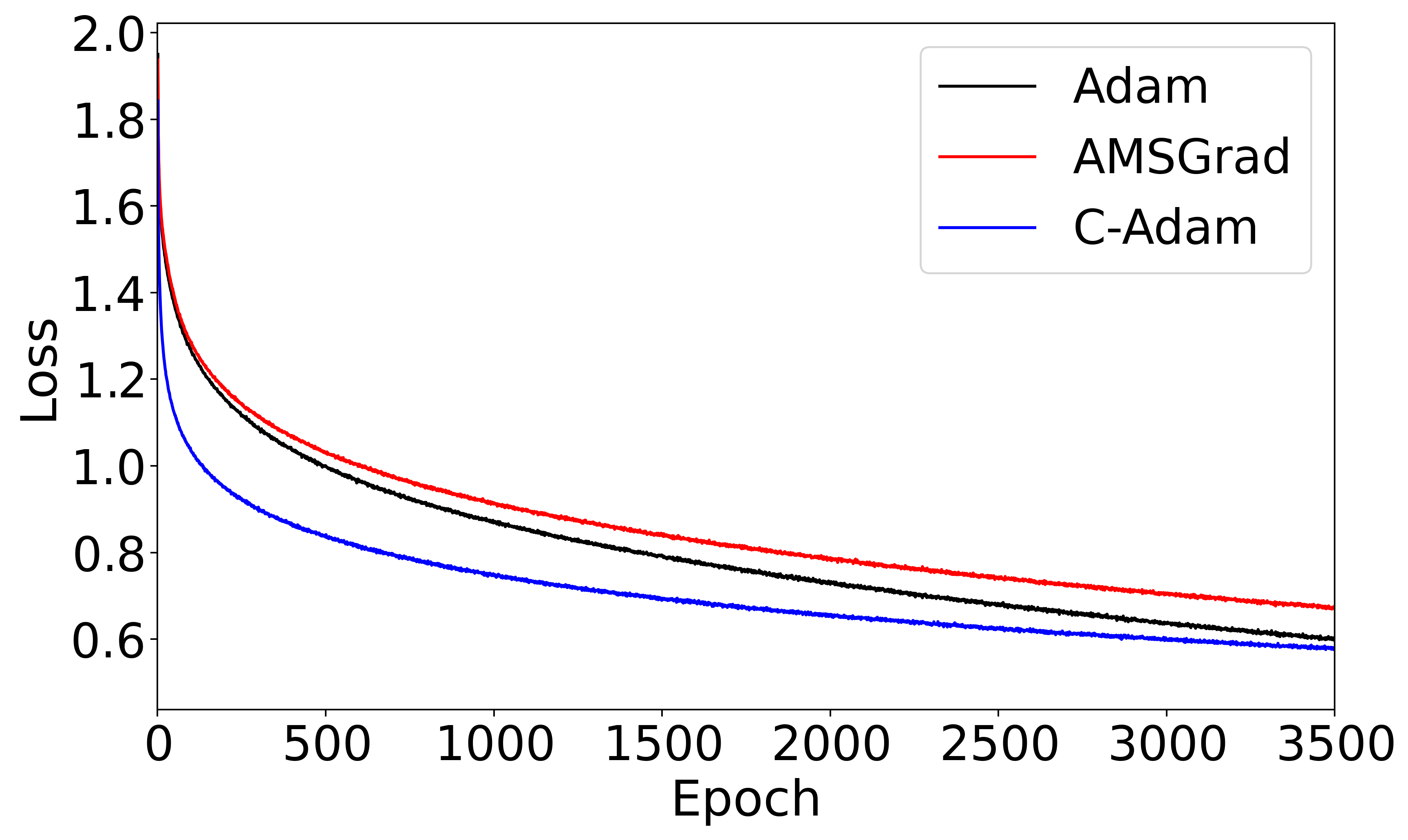}
        \end{overpic}
        \vspace{0.1cm}
        {\small (a) Training}
    \end{subfigure}
    \hspace{0.1cm}
    \begin{subfigure}[t]{0.48\textwidth}
        \centering
        \begin{overpic}[width=\linewidth]{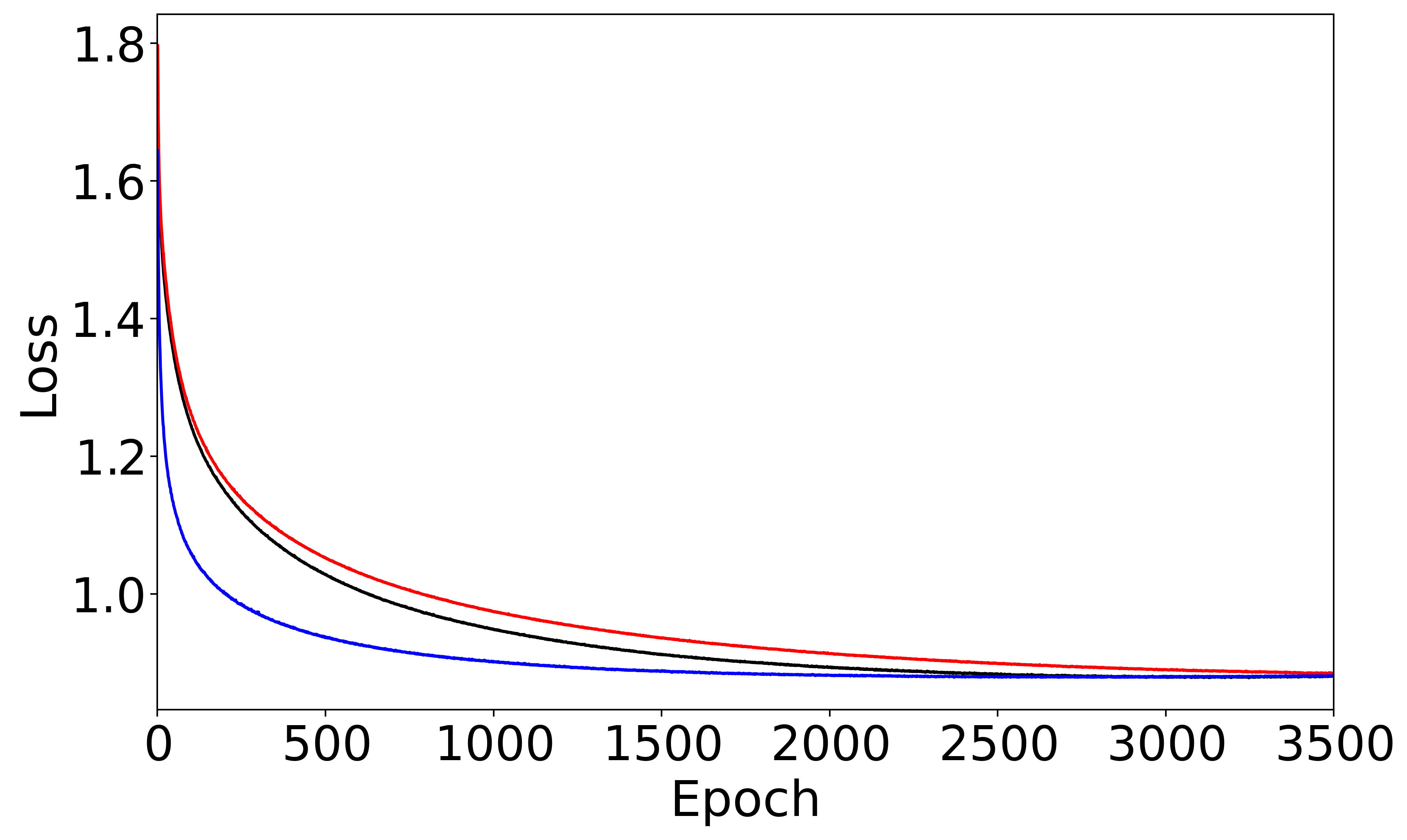}
        \end{overpic}
        \vspace{0.1cm}
        {\small (b) Validation}
    \end{subfigure}

    \caption{Training and validation losses of the three optimizers for multiclass classification over CIFAR-10 using Convolutional Neural Network \ref{cnn section}}
\label{fig:adam_vs_amsgrad_vs_customadam_cnn}
\end{figure}

\begin{figure}[]
    \centering

    \begin{subfigure}[t]{0.51\textwidth}
        \centering
        \begin{overpic}[width=\linewidth]{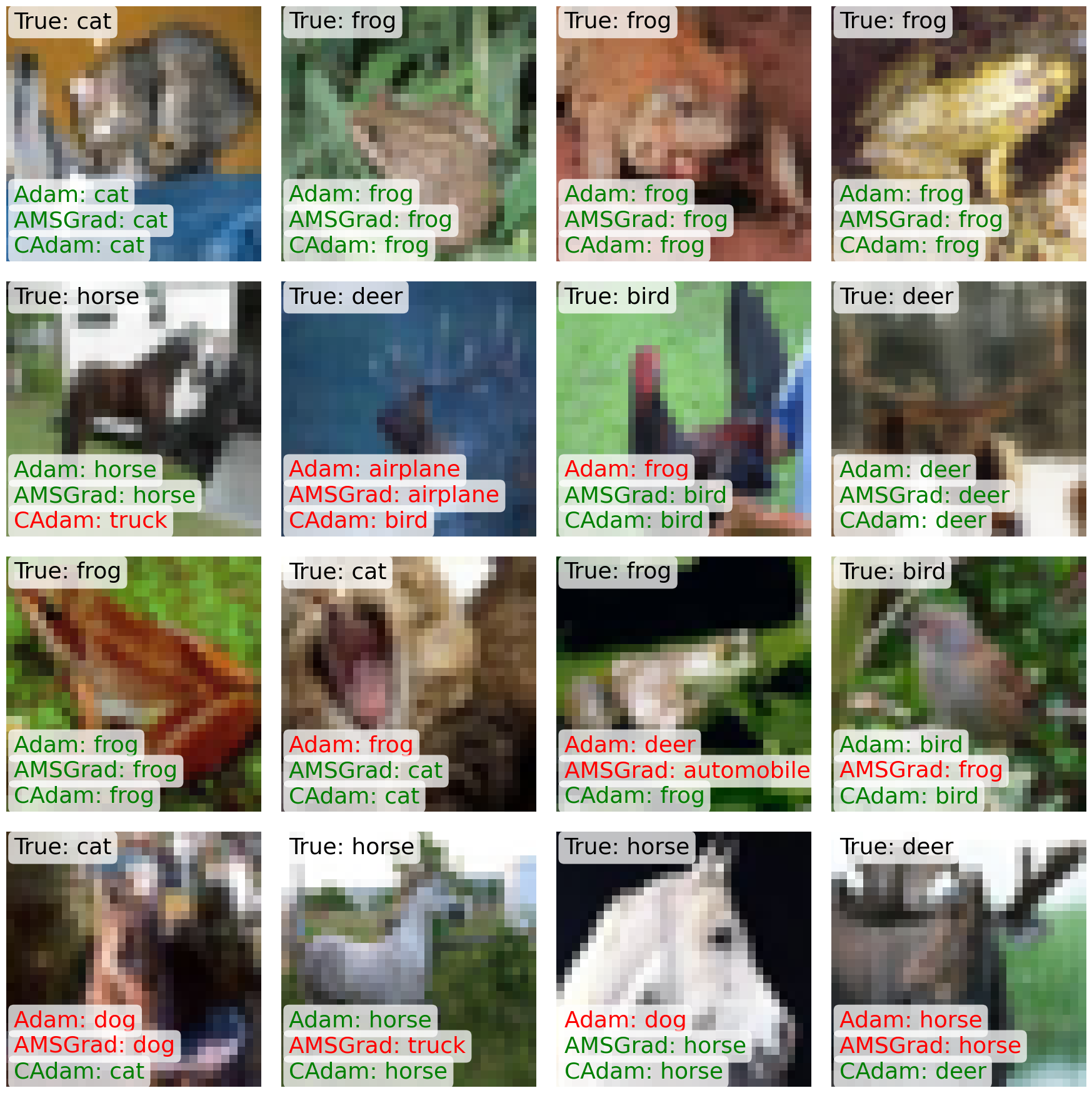}
        \end{overpic}
        \vspace{0.1cm}
        {\small (a) Prediction}
    \end{subfigure}
    \hspace{0.1cm}
    \begin{subfigure}[t]{0.45\textwidth}
        \centering
        \begin{overpic}[width=\linewidth]{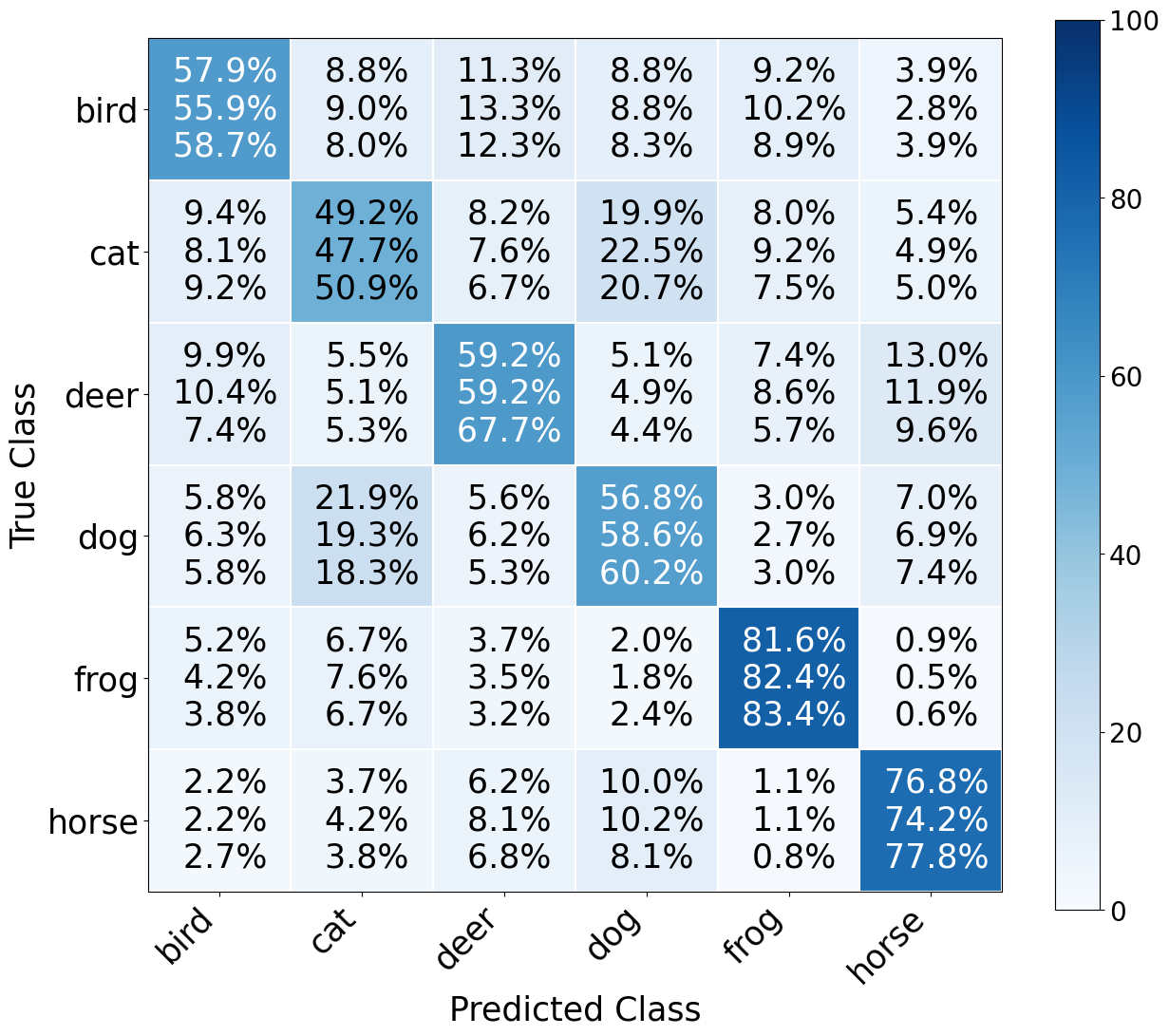}
        \end{overpic}
        \vspace{0.1cm}
        {\small (b) Confusion matrix}
    \end{subfigure}

    \caption{(a) Prediction results across randomly chosen CIFAR-10 dataset. For each image, the true label is shown in the top-left corner, and the prediction is shown in the bottom-left corner. The true predictions are marked in green and false predictions are marked in red. (b) Corresponding confusion matrix for the 5000 images in the CIFAR-10 dataset. The three percentages in each block correspond to Adam, AMSGrad and C-Adam, respectively. For better readability, only the animal classes are included in (a) and (b).}
\label{fig:predictons_images_cm}
\end{figure}

To demonstrate the advantage of the early-stage convergence of C-Adam, the best model obtained after training for a maximum of 1000 epochs across the three optimizers is used to predict on the CIFAR-10 dataset (Fig. \ref{fig:predictons_images_cm}a). The corresponding confusion matrix is also shown in Fig. \ref{fig:predictons_images_cm}(b). From Figs. \ref{fig:predictons_images_cm}(a) $\&$ (b), it can be inferred that C-Adam achieves a clear improvement over Adam and AMSGrad by reaching a better solution early, suggesting that the proposed optimizer is particularly beneficial in limited-training regimes. When training is extended to 3500 epochs, the performance and accuracy of all three optimizers are comparable, as the validation loss converges at higher epochs (Fig. \ref{fig:adam_vs_amsgrad_vs_customadam_cnn}b). Specifically, when the training is extended from 1000 epochs to 3500 epochs, the accuracy is improved by 5.08, 6.56 and 2.34 $\%$ for Adam, AMSGrad and C-Adam, respectively, with a smaller improvement on C-Adam, which already achieved better accuracy at 1000 epochs and continued to retrain relatively higher accuracy at higher epochs. It is noteworthy that the relatively low accuracy and prediction percentage in the confusion matrix on the CIFAR-10 dataset are due to the simple CNN architecture considered in the current study. In future work, the proposed optimizer will be trained and tested on more complex CNN architectures. 

\begin{rem}
   In the case where $v_{t}$, $\forall\hspace{0.1cm} t\in [T]$ is small, C-Adam yields the maximum of past values of second moment to be small as well. This increases the adaptive learning rate and results in oscillations. We can solve this issue by adding a thresholding in the algorithm of C-Adam. After finding the best $\epsilon_{0}$ using cross validation, a slightly modified version of C-Adam, which we call C-Adam\_V2 here, changes $\max(\tilde{v}_{t-1}, v_t)$ to $\max(\max(\tilde{v}_{t-1}, v_t), \epsilon_{0})$. This yields faster convergence and better results for the synthetic problem (\ref{syn exp section}) as shown in Fig.~\ref{fig:synthetic_v2_cadam}(a). 
    Since this change requires the knowledge of a suitable $\epsilon_{0}$ beforehand, it becomes difficult to apply this method for real world complex dataset problems and models where predicting a ``good" $\epsilon_0$ is not feasible. If $\epsilon_{0}$ is known for these problems, then C-Adam\_V2 will outperform Adam, AMSGrad and C-Adam in terms of convergence (Fig.~\ref{fig:synthetic_v2_cadam}a) and regret (Fig.~\ref{fig:synthetic_v2_cadam}b).
\end{rem}

\begin{figure}[]
    \centering
  \begin{subfigure}[t]{0.48\textwidth}
        \centering
        \begin{overpic}[width=\linewidth]{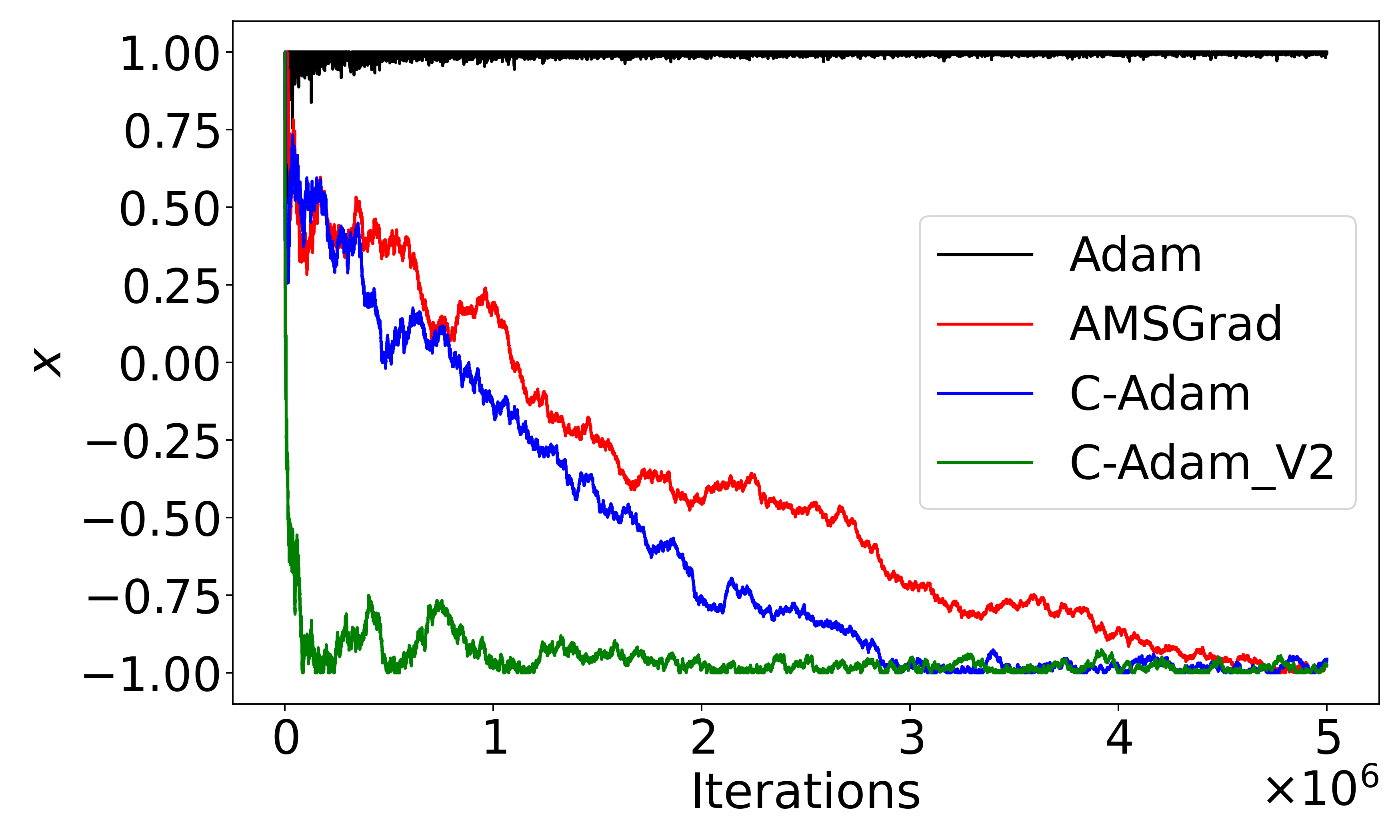}
        \end{overpic}
        \vspace{0.1cm}
        {\small (a) x versus iterations}
    \end{subfigure}
    \hspace{0.1cm}
    \begin{subfigure}[t]{0.48\textwidth}
        \centering
        \begin{overpic}[width=\linewidth]{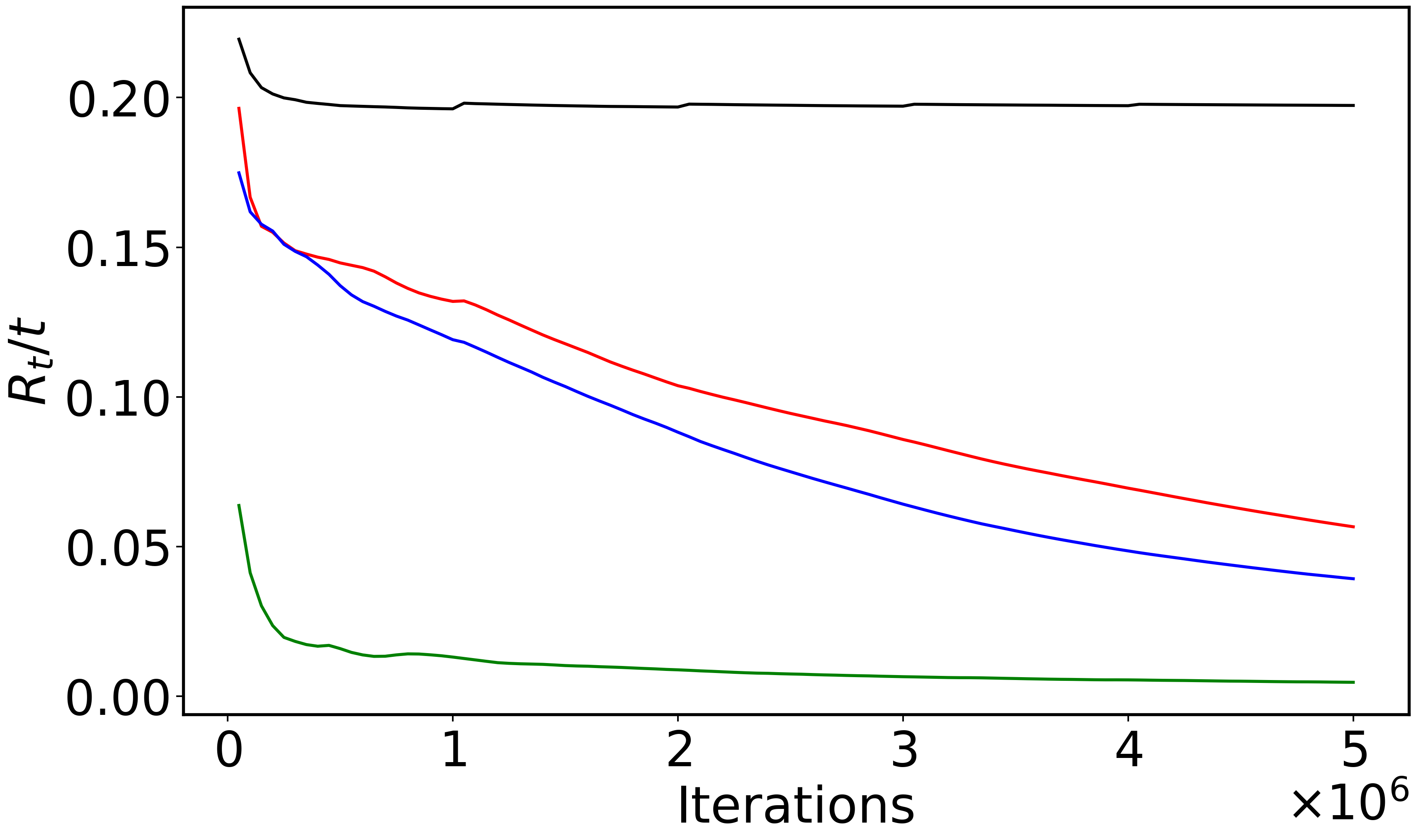}
        \end{overpic}
        \vspace{0.1cm}
        {\small (b) regret bound versus iterations}
    \end{subfigure}

    \caption{Comparison of Adam, AMSGrad, C-Adam and C-Adam\_V2 for Synthetic problem \ref{syn exp section}}
\label{fig:synthetic_v2_cadam}
\end{figure}

\section{Conclusion}\label{conclusion}
The study of adaptive optimization methods has evolved significantly with the introduction of theoretical analyses of AMSGrad. In the same domain, we have proposed an optimizer, C-Adam, which tends to perform better in terms of speed, convergence and lesser oscillations in comparison to Adam and AMSGrad. For noisy datasets where both Adam and AMSGrad fail to yield satisfactory results, we have demonstrated that C-Adam significantly outperforms them on the classification task. Theoretical analysis of the regret bound of the optimizer has been studied in detail. The proposed algorithm worked faster and with less oscillations on a synthetic example in which Adam failed to locate the optimal point and AMSGrad inspite of reaching the optimal point, did so in longer time and with high oscillations. Further, the proposed algorithm was validated on the multiclass classification on the MNIST dataset for logistic regression. C-Adam also performs experimentally better than Adam and AMSGrad for image classification on the CIFAR-10 dataset. Hence, the line-of-sight approach used in C-Adam seems to have worked well across a variety of examples giving faster convergence.

\bibliographystyle{unsrt}
\bibliography{ref}
\end{document}